\documentclass{article}
\pdfoutput=1
\usepackage{subfigure}
\usepackage{rotating}
\usepackage{algorithm}
\usepackage{algorithmic}
\usepackage{multirow}
\usepackage{bibentry}
\usepackage{epstopdf} 
\usepackage{enumitem}
\usepackage{url}
\setlist{nolistsep}
\usepackage[titletoc,toc,title]{appendix}
\usepackage{amsmath}
\usepackage{indentfirst}
\usepackage{amssymb}
\usepackage{amsfonts}
\usepackage{amsthm}
\usepackage{comment}
\usepackage{color}
\usepackage{verbatim}
\usepackage{bm}
\usepackage{bbm}

\newcommand{\sign}{\operatornamewithlimits{sign}}

\newtheorem{theorem}{Theorem}

\newtheorem{corollary}{Corollary}
\newtheorem{definition}{Definition}

\newtheorem{lemma}{Lemma}
\newcommand{\cO}{\mathcal{O}}

\newcommand{\bS}{\mathbf{S}}
\newcommand{\RFF}{\text{RFF}}
\newcommand{\SORF}{\text{SORF}}
\newcommand{\ORF}{\text{ORF}}
\newcommand{\bR}{\mathbf{R}}
\newcommand{\bW}{\mathbf{W}}
\newcommand{\bG}{\mathbf{G}}
\newcommand{\bg}{\mathbf{g}}

\newcommand{\textprime}{$^\prime$}

\newcommand{\bH}{\mathbf{H}}
\newcommand{\bD}{\mathbf{D}}

\newcommand{\bA}{\mathbf{A}}
\newcommand{\Var}{\text{Var}}
\newcommand{\bQ}{\mathbf{Q}}

\newcommand{\bx}{\mathbf{x}}

\newcommand{\bw}{\mathbf{w}}

\newcommand{\bu}{\mathbf{u}}
\newcommand{\bv}{\mathbf{v}}
\newcommand{\norm}[1]{\lVert #1 \rVert}

\newcommand{\E}{\mathbb{E}}

\newcommand{\by}{\mathbf{y}}
\newcommand{\bz}{\mathbf{z}}

\renewcommand{\[}{\begin{equation*}}
\renewcommand{\]}{\end{equation*}}

\usepackage[final, nonatbib]{nips_2016}

\usepackage[utf8]{inputenc} 
\usepackage[T1]{fontenc}    
\usepackage{hyperref}       
\usepackage{url}            
\usepackage{booktabs}       
\usepackage{amsfonts}       
\usepackage{nicefrac}       
\usepackage{microtype}      

\title{Orthogonal Random Features}

\author{
Felix Xinnan Yu \,\, Ananda Theertha Suresh \,\, Krzysztof Choromanski \\ \textbf{Daniel Holtmann-Rice \,\, Sanjiv Kumar}\\ \\
Google Research, New York \\
\{\texttt{felixyu}, \texttt{theertha}, \texttt{kchoro}, \texttt{dhr}, \texttt{sanjivk}\}\texttt{@google.com}
}

\begin{document}
\maketitle
\begin{abstract}
\vspace{-0.2cm}
We present an intriguing discovery related to Random Fourier Features: in Gaussian kernel approximation, 
replacing the random Gaussian matrix by
a properly scaled random orthogonal matrix significantly decreases kernel approximation error. We call this technique Orthogonal Random Features (ORF), and provide theoretical and empirical justification for this behavior.
Motivated by this discovery, we further propose Structured Orthogonal Random Features (SORF), which uses a class of structured discrete orthogonal matrices to speed up the computation. The method reduces the time cost from $\mathcal{O}(d^2)$ to $\mathcal{O}(d \log d)$, where $d$ is the data dimensionality, with almost no compromise in kernel approximation quality compared to ORF. Experiments on several datasets verify the effectiveness of ORF and SORF over the existing methods. We also provide discussions on using the same type of discrete orthogonal structure for a broader range of applications. 
\end{abstract}

\section{Introduction}

Kernel methods are widely used in nonlinear learning \cite{cortes1995support}, but they are computationally expensive for large datasets. Kernel approximation is a powerful technique to make kernel methods scalable, by mapping input features into a new space where dot products approximate the kernel well~\cite{rahimi2007random}. With accurate kernel approximation, efficient linear classifiers can be trained in the transformed space while retaining the expressive power of nonlinear methods \cite{joachims2006training, shalev2011pegasos}.

Formally, given a kernel $K(\cdot, \cdot): \mathbb{R}^d \times \mathbb{R}^d \rightarrow \mathbb{R}$, kernel approximation methods seek to find a nonlinear transformation $\phi(\cdot): \mathbb{R}^d \rightarrow \mathbb{R}^{d'}$ such that, for any $\bx, \by \in \mathbb{R}^d$
\[
K(\bx, \by) \approx \hat{K}(\bx, \by) = \phi(\bx)^T \phi(\by).
\]

Random Fourier Features \cite{rahimi2007random} are used widely in approximating smooth, shift-invariant kernels. This technique requires the kernel to exhibit two properties: 1) shift-invariance, \emph{i.e.} $K(\bx,\by) = K(\Delta)$ where $\Delta = \bx - \by$; and 2) positive semi-definiteness of $K(\Delta)$ on $\mathbb{R}^d$. The second property guarantees that the Fourier transform of $K(\Delta)$ is a nonnegative function \cite{bochner1955harmonic}. Let $p(\bw)$ be the Fourier transform of $K(\bz)$. Then,
\[
K({\bx- \by}) = \int_{\mathbb{R}^d}  p(\bw)e^{j\bw^T(\bx - \by)} d\bw.
\]
This means that one can treat $p(\bw)$ as a density function and use Monte-Carlo sampling to derive the following nonlinear map for a real-valued kernel:
\[
\phi(\bx) = \sqrt{1/D}  \big[ \sin (\bw_1^T \bx), \cdots, \sin(\bw_{D}^T\bx), \cos (\bw_1^T \bx), \cdots, \cos(\bw_{D}^T\bx) \big]^T,
\]
where $\bw_i$ is sampled i.i.d.\ from a probability distribution with density $p(\bw)$.
Let $\bW = \big[\bw_1, \cdots, \bw_D \big]^T$. The linear transformation $\bW\bx$ is central to the above computation since, 
\begin{itemize}[leftmargin=*,label=\raisebox{0.25ex}{\tiny$\bullet$}]
    \item The choice of matrix $\bW$ determines how well the estimated kernel converges to the actual kernel;
	\item The computation of $\bW\bx$ has space and time costs of $\mathcal{O}(Dd)$. This is expensive for high-dimensional data, especially since $D$ is often required to be larger than $d$ to achieve low approximation error.
\end{itemize}

\begin{figure}
\centering
\subfigure[\texttt{USPS}]{\includegraphics[width=0.33\textwidth]{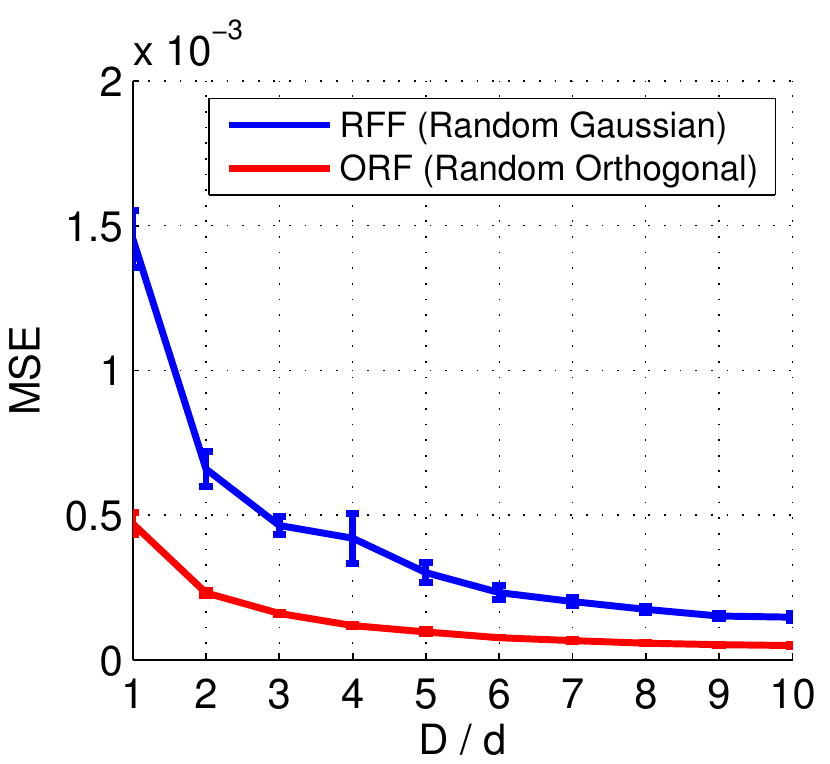}}
\hspace{-0.1cm}
\subfigure[\texttt{MNIST}]{\includegraphics[width=0.33\textwidth]{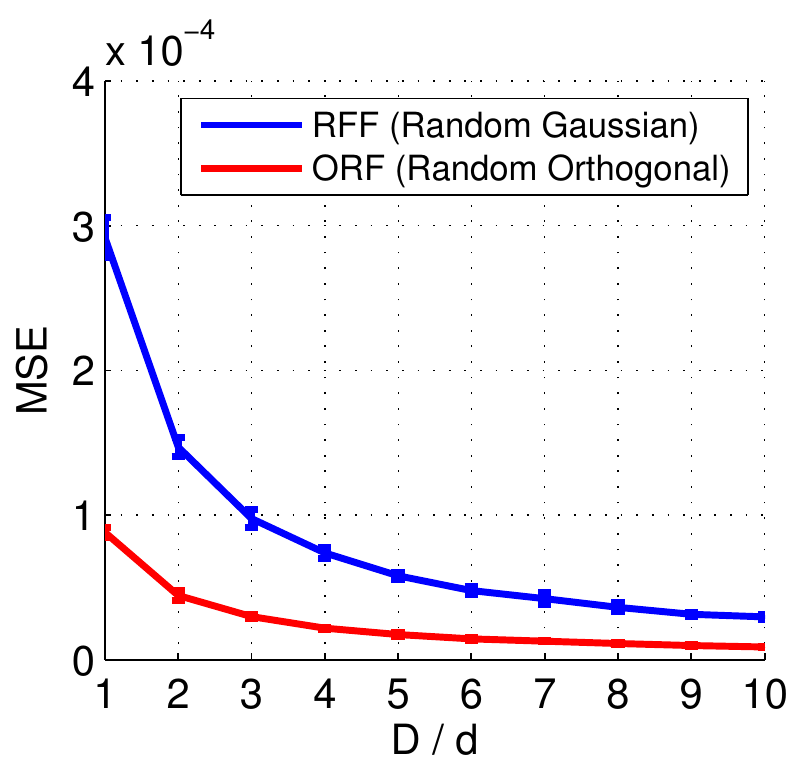}}
\hspace{-0.1cm}
\subfigure[\texttt{CIFAR}]{\includegraphics[width=0.33\textwidth]{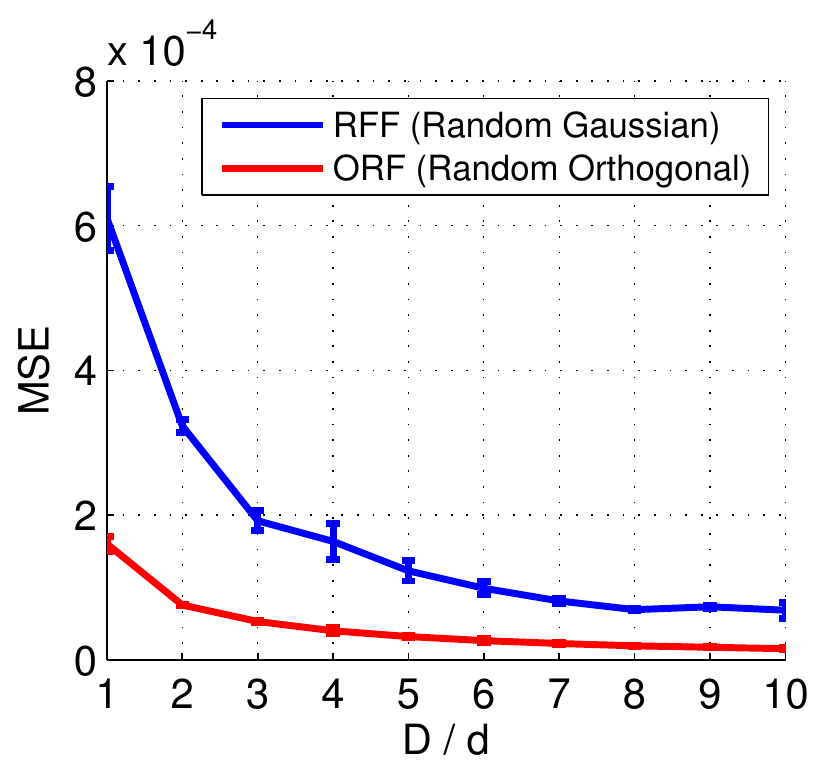}}
\label{fig:motivation}
\vspace{-0.3cm}
\caption{Kernel approximation mean squared error (MSE) for the Gaussian kernel $K(\bx, \by) = e^{-||\bx - \by ||^2 / 2\sigma^2}$. $D$: number of rows in the linear transformation $\bW$. $d$: input dimension. ORF imposes orthogonality on $\bW$ (Section \ref{sec:ORF}).}
\vspace{-0.1cm}
\end{figure}

In this work, we address both of the above issues. We first show an intriguing discovery (Figure \ref{fig:motivation}):  by enforcing orthogonality on the rows of $\bW$, the kernel approximation error can be significantly reduced. We call this method Orthogonal Random Features (ORF). Section \ref{sec:ORF} describes the method and provides theoretical explanation for the improved performance.

Since both generating a $d \times d$ orthogonal matrix ($\mathcal{O}(d^3)$ time and $\mathcal{O}(d^2)$ space) and computing the transformation ($\mathcal{O}(d^2)$ time and space) are prohibitively expensive for high-dimensional data, we further propose Structured Orthogonal Random Features (SORF) in Section \ref{sec:SORF}. The idea is to replace random orthogonal matrices by a class of special structured matrices consisting of products of binary diagonal matrices and Walsh-Hadamard matrices. SORF has fast computation time, $\mathcal{O}(D \log d)$, and almost no extra memory cost (with efficient in-place implementation). We show extensive experiments in Section \ref{sec:exp}. 
We also provide theoretical discussions in Section \ref{sec:extensions} of applying the structured matrices in a broader range of applications where random Gaussian matrix is used.

\section{Related Works}
\label{sec:related}
Explicit nonlinear random feature maps have been constructed for many types of kernels, such as intersection kernels \cite{maji2009max}, generalized RBF kernels \cite{sreekanth2010generalized}, skewed multiplicative histogram kernels \cite{li2010random}, additive kernels \cite{vedaldi2012efficient}, and polynomial kernels \cite{kar2012random,pennington2015spherical}. 
In this paper, we focus on approximating Gaussian kernels following the seminal Random Fourier Features (RFF) framework \cite{rahimi2007random}, which has been extensively studied both theoretically and empirically \cite{yang2012nystrom, rudi2016generalization, sriperumbudur2015optimal}.

Key to the RFF technique is Monte-Carlo sampling. It is well known that the convergence of Monte-Carlo can be largely improved by carefully choosing a deterministic sequence instead of random samples \cite{niederreiter2010quasi}. Following this line of reasoning, Yang et al. \cite{yang2014quasi} proposed to use low-displacement rank sequences in RFF. Yu et al. \cite{arxiv_cnm} studied optimizing the sequences in a data-dependent fashion to achieve more compact maps. 
In contrast to the above works, this paper is motivated by an intriguing new discovery that using orthogonal random samples provides much faster convergence. 
Compared to \cite{yang2014quasi}, the proposed SORF method achieves both lower kernel approximation error and greatly reduced computation and memory costs. Furthermore, unlike \cite{arxiv_cnm}, the results in this paper are data independent. 

Structured matrices have been used for speeding up dimensionality reduction \cite{ailon2006approximate}, binary embedding \cite{icml14_cbe}, deep neural networks \cite{cheng2015exploration} and kernel approximation \cite{le2013fastfood, arxiv_cnm, choromanski2016recycling}. 
For the kernel approximation works, in particular, the ``structured randomness'' leads to a minor loss of accuracy, but allows faster computation since the structured matrices enable the use of FFT-like algorithms. Furthermore, these matrices provide substantial model compression since they require subquadratic (usually only linear) space. In comparison with the above works, our proposed methods SORF and ORF are more effective than RFF. In particular SORF demonstrates \emph{both} lower approximation error and better efficiency than RFF. Table \ref{table:comparison} compares the space and time costs of different techniques. 

\begin{table}[t]
\small
\centering
\begin{tabular}{l|l|l|l}
\hline Method & Extra Memory & Time & Lower error than RFF? \\
\hline Random Fourier Feature (RFF) \cite{rahimi2007random} & $\mathcal{O}(Dd)$  & $\mathcal{O}(Dd)$ &  - \\ 
\hline Compact Nonlinear Map (CNM) \cite{arxiv_cnm} & $\mathcal{O}(Dd)$ &  $\mathcal{O}(Dd)$ & Yes (data-dependent)\\
\hline Quasi-Monte Carlo (QMC) \cite{yang2014quasi} & $\mathcal{O}(Dd)$ & $\mathcal{O}(Dd)$ & Yes \\ 
\hline Structured (fastfood/circulant) \cite{arxiv_cnm,le2013fastfood} & $\mathcal{O}(D)$ & $\mathcal{O}(D \log d)$ & No\\ 
\hline \textbf{Orthogonal Random Feature (ORF)} & \bm{$\mathcal{O}(Dd)$} & \bm{$\mathcal{O}(Dd)$} & \textbf{Yes} \\ 
\hline \textbf{Structured ORF (SORF)} & \bm{$\mathcal{O}(D)$} or \bm{$\mathcal{O}(1)$} & \bm{$\mathcal{O}(D\log d)$} & \textbf{Yes}\\ 
\hline 
\end{tabular}
\vspace{+0.2cm}
\caption{Comparison of different kernel approximation methods under the framework of Random Fourier Features \cite{rahimi2007random}. We assume $D \geq d$. The proposed SORF method have $\mathcal{O}(D)$ degrees of freedom. The computations can be efficiently implemented as in-place operations with fixed random seeds. Therefore it can cost $\mathcal{O}(1)$ in extra space.}
\vspace{-0.5cm}
\label{table:comparison}
\end{table}

\section{Orthogonal Random Features}
\label{sec:ORF}
Our goal is to approximate a Gaussian kernel of the form
\[
K(\bx, \by) = e^{-||\bx - \by ||^2 / 2 \sigma^2}.
\]
In the paragraph below, we assume a square linear transformation matrix $\bW \in \mathbb{R}^{D \times d}$, $D = d$. When $D < d$, we simply use the first $D$ dimensions of the result. When $D > d$, we use multiple independently generated random features and concatenate the results. We comment on this setting at the end of this section. 
 
Recall that the linear transformation matrix of RFF can be written as
\begin{equation}
	\bW_{\textrm{RFF}} = \frac{1}{\sigma} \mathbf{\bG},
\label{eq:rff}
\end{equation}
where $\mathbf{G} \in \mathbb{R}^{d \times d}$ is a random Gaussian matrix, with every entry sampled independently from the standard normal distribution. 
Denote the approximate kernel based on the above $\bW_{\text{RFF}}$ as $K_{\text{RFF}}(\bx, \by)$. For completeness, we first show the expectation and variance of $K_{\text{RFF}}(\bx, \by)$.
\begin{lemma}(Appendix~\ref{app:RFF})
\label{lem:RFF}
$K_{\text{RFF}}(\bx, \by)$ is an unbiased estimator of the Gaussian kernel, i.e.,
$
\mathbb{E}(K_{\text{RFF}}(\bx, \by)) = e^{-||\bx - \by ||^2 / 2 \sigma^2}.
$
Let $z = ||\bx - \by|| / \sigma$. The variance of $K_{\text{RFF}}(\bx, \by)$ is 
$
\Var\left(K_{\text{RFF}}(\bx, \by) \right) = \frac{1}{2D} 
\left(1-e^{-z^2}\right)^2.
$
\end{lemma}

The idea of Orthogonal Random Features (ORF) is to impose orthogonality on the matrix on the linear transformation matrix $\mathbf{G}$.
Note that one cannot achieve unbiased kernel estimation by simply replacing $\bG$ by an orthogonal matrix, since the norms of the rows of $\bG$ follow the $\chi$-distribution, while rows of an orthogonal matrix have the unit norm. The linear transformation matrix of ORF has the following form
\begin{equation}
	\bW_{\textrm{ORF}} = \frac{1}{\sigma} \bS \bQ,
	\label{eq:ORF}
\end{equation}
where $\bQ$ is a uniformly distributed random orthogonal matrix\footnote{We first generate the random Gaussian matrix $\bG$ in (\ref{eq:rff}). $\bQ$ is the orthogonal matrix obtained from the QR decomposition of $\bG$. $\bQ$ is distributed uniformly on the Stiefel manifold (the space of all orthogonal matrices) based on the Bartlett decomposition theorem \cite{muirhead2009aspects}.}. 
The set of rows of $\bQ$ forms a bases in $\mathbb{R}^d$. $\bS$ is a diagonal matrix, with diagonal entries sampled i.i.d.\ from the $\chi$-distribution with $d$ degrees of freedom. $\bS$ makes the norms of the rows of $\bS \bQ$ and $\mathbf{G}$ identically distributed.

\begin{figure}
\centering
\subfigure[Variance ratio (when $d$ is large)]{\includegraphics[width=0.32\textwidth]{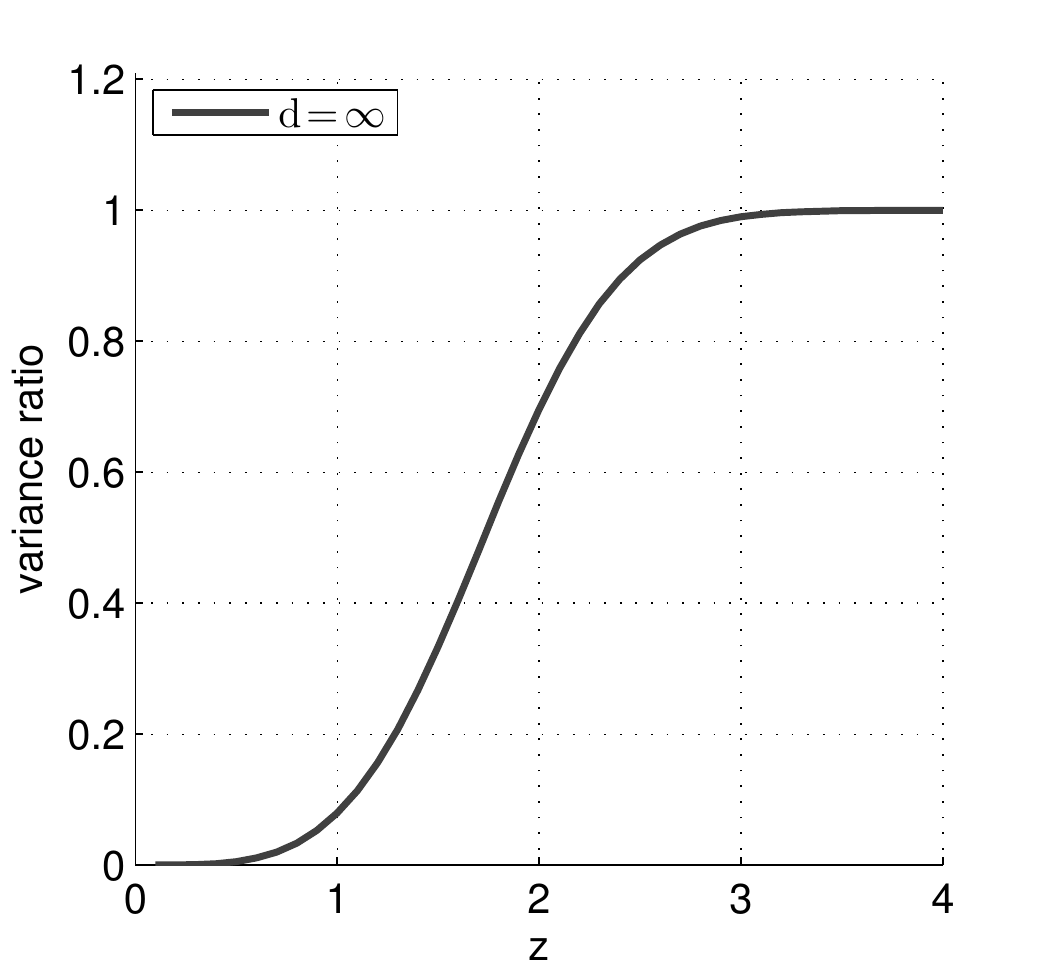}}
\hspace{-0.1cm}
\subfigure[Variance ratio (simulation)]{\includegraphics[width=0.33\textwidth]{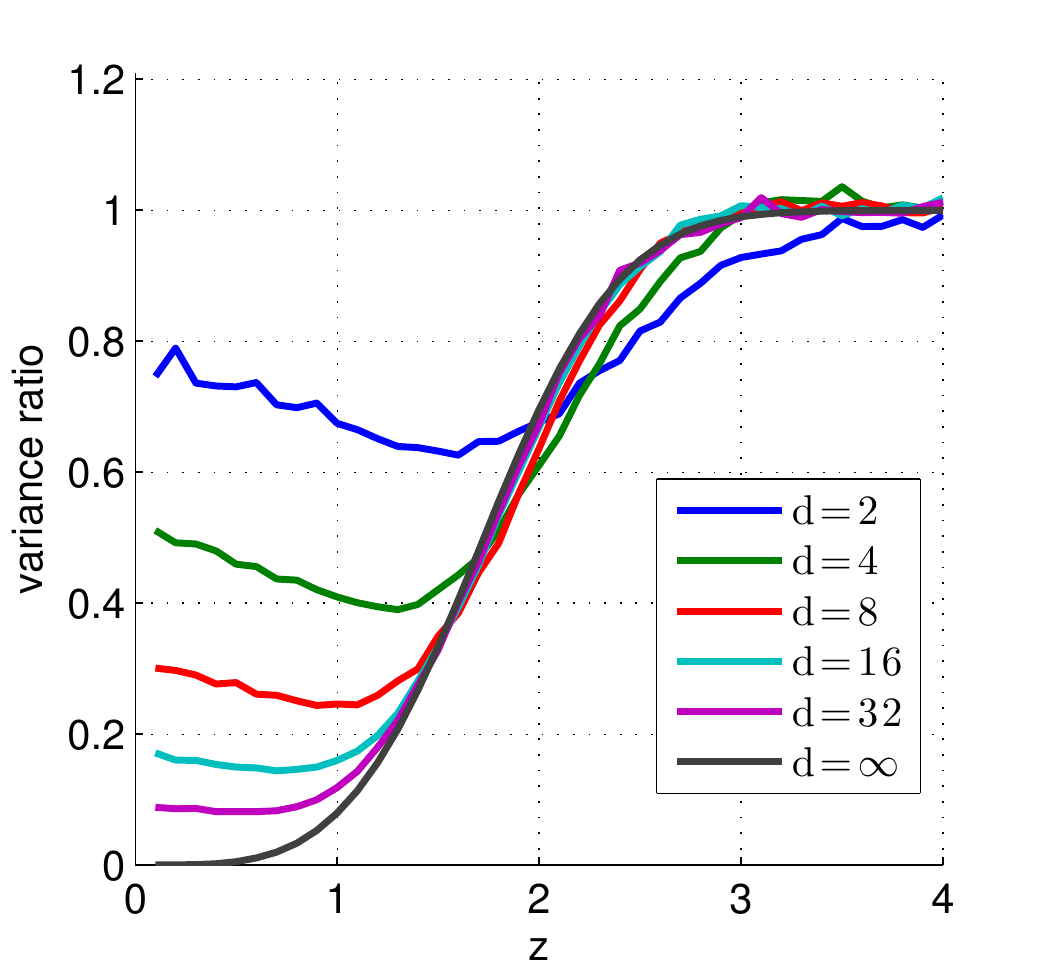}}
\hspace{-0.1cm}
\subfigure[Empirical distribution of $z$]{\includegraphics[width=0.33\textwidth]{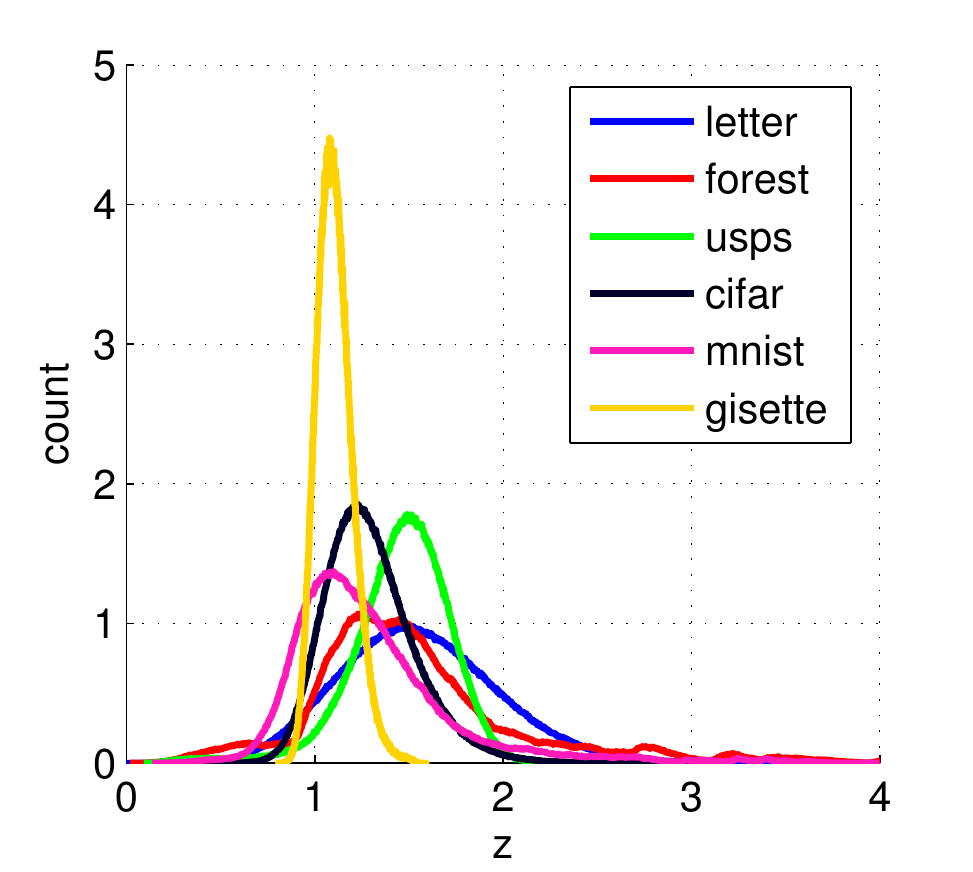}}
\caption{(a) $\Var(K_{\ORF}(\bx, \by)) / \Var(K_{\RFF}(\bx, \by))$ when $d$ is large and $d = D$. $z = ||\bx - \by|| / \sigma$. 
(b) Simulation of $\Var(K_{\ORF}(\bx, \by)) / \Var(K_{\RFF}(\bx, \by))$ when $D = d$. Note that the empirical variance is the Mean Squared Error (MSE). (c) Distribution of $z$ for several datasets, when we set $\sigma$ as the mean distance to 50th-nearest neighbor for samples from the dataset. The count is normalized such that the area under curve for each dataset is 1. Observe that most points in all the datasets have $z < 2$. As shown in (a), for these values of $z$, ORF has much smaller variance compared to the standard RFF.} 
\label{fig:variance}
\vspace{-0.2cm}
\end{figure}

Denote the approximate kernel based on the above $\bW_{\textrm{ORF}}$ as $K_{\text{ORF}}(\bx, \by)$. The following shows that $K_{\text{ORF}}(\bx, \by)$ is an unbiased estimator of the kernel, and it has lower variance in comparison to RFF. 
\begin{theorem}
$K_{\text{ORF}}(\bx, \by)$ is an unbiased estimator of the Gaussian kernel, i.e.,
\[
\mathbb{E}(K_{\text{ORF}}(\bx, \by)) = e^{-||\bx - \by ||^2 / 2 \sigma^2}.
\]
Let $D \leq d$, and $z = ||\bx - \by|| / \sigma$. There exists a function $f$ such that for all $z$, the variance of $K_{\text{ORF}}(\bx, \by)$ is bounded by
\vspace{-0.1cm}
\[
\Var\left(K_{\text{ORF}}(\bx, \by) \right) \leq \frac{1}{2D} \left(\left(1-e^{-z^2}\right)^2 - \frac{D-1}{d}  e^{-z^2} z^4 \right) + \frac{f(z)}{d^2}.
\]
\vspace{-0.5cm}
\end{theorem}

\begin{proof}
We first show the proof of the unbiasedness.
Let $\bz = \frac{\bx - \by}{\sigma}$, and $z = ||\bz||$, then 
$\mathbb{E}({K}_{ORF}(\bx, \by)) 
= 
\mathbb{E}\left(\frac{1}{D} \sum_{i=1}^D \cos(\bw_i^T \bz)\right)
= \frac{1}{D} \sum_{i=1}^D \mathbb{E} \left( \cos(\bw_i^T \bz)\right)$. 
Based on the definition of ORF, $\bw_1, \bw_2, \dots,  \bw_D$ are $D$ random vectors given by
$
\bw_i = s_i \bu_i,
$
with $\bu_1, \bu_2, \ldots, \bu_d$ a uniformly chosen random orthonormal basis for $\bR^d$, 
and $s_i$'s are independent $\chi$-distributed random variables with $d$ degrees of freedom.
It is easy to show that for each $i$, $\bw_i$ is distributed according to $N(0, \mathbf{I}_d)$, and hence by Bochner's theorem,
\[
\E[\cos(\bw^T \bz)] = e^{-z^2/2}.
\]
We now show a proof sketch of the variance.
Suppose, $a_i = \cos(\bw^T_i \bz)$.
\begin{small}
\begin{align*}
\Var\left(\frac{1}{D}\sum^D_{i=1} a_i\right) 
& = \E\left[\left( \frac{\sum^D_{i=1} a_i}{D} \right)^2\right] - \E\left[\left( \frac{\sum^D_{i=1} a_i}{D} \right)\right]^2 \\
& = \frac{1}{D^2}\sum_i \left(\E[a^2_i]  - \E[a_i]^2\right) +  \frac{1}{D^2}\sum_i \sum_{j \neq i} \left(\E[a_i a_j]  - \E[a_i]\E[a_j]\right) \\
& = \frac{\left(1-e^{-z^2}\right)^2}{2D} + \frac{D(D-1)}{D^2} \left(\E[a_1 a_2]  - e^{-z^2}\right),
\end{align*}
\end{small}
where the last equality follows from symmetry. The first term in the resulting expression is exactly the variance of RFF. 
In order to have lower variance, $\E[a_1 a_2]  - e^{-z^2}$ must be negative. We use the following lemma to quantify this term.
\begin{lemma}(Appendix~\ref{app:neg_cor})
\label{lem:neg_cor}
There is a function $f$ such that 
for any $z$, 
\[
\E[a_i a_j] \leq e^{-z^2} - e^{-z^2} \frac{z^4}{2d} +  \frac{f(z)}{d^2}. 
\]
\vspace{-1.5cm}
\end{lemma}
\end{proof}
Therefore, for a large $d$, and $D \leq d$,  the ratio of the variance of ORF and RFF is
\begin{equation}
\frac{\text{Var}(K_{\ORF}(\bx, \by))} {\text{Var}(K_{\text{RFF}}(\bx, \by))} \approx	1 - \frac {(D-1) e^{-z^2} z^4} {d (1 - e^{-z^2})^2}.
\label{eq:ratio}
\end{equation}
Figure \ref{fig:variance}(a) shows the ratio of the variance of ORF to that of RFF when $D = d$ and $d$ is large. 
First notice that this ratio is always smaller than 1, and hence ORF always provides improvement over the conventional RFF. 
Interestingly, we gain significantly for small values of $z$. 
In fact, when $z \rightarrow 0$ and $d \rightarrow \infty$, the ratio is roughly $z^2$ (note $e^x \approx 1 + x$ when $x \rightarrow 0$), and ORF exhibits infinitely lower error relative to RFF. Figure \ref{fig:variance}(b) shows empirical simulations of this ratio. We can see that the variance ratio is close to that of $d=\infty$ (\ref{eq:ratio}), even when $d = 32$, a fairly low-dimensional setting in real-world cases.

\begin{figure}
\centering
\hspace{-0.27cm}	
\subfigure[Bias of ORF\textprime]{\includegraphics[width=0.26\textwidth]{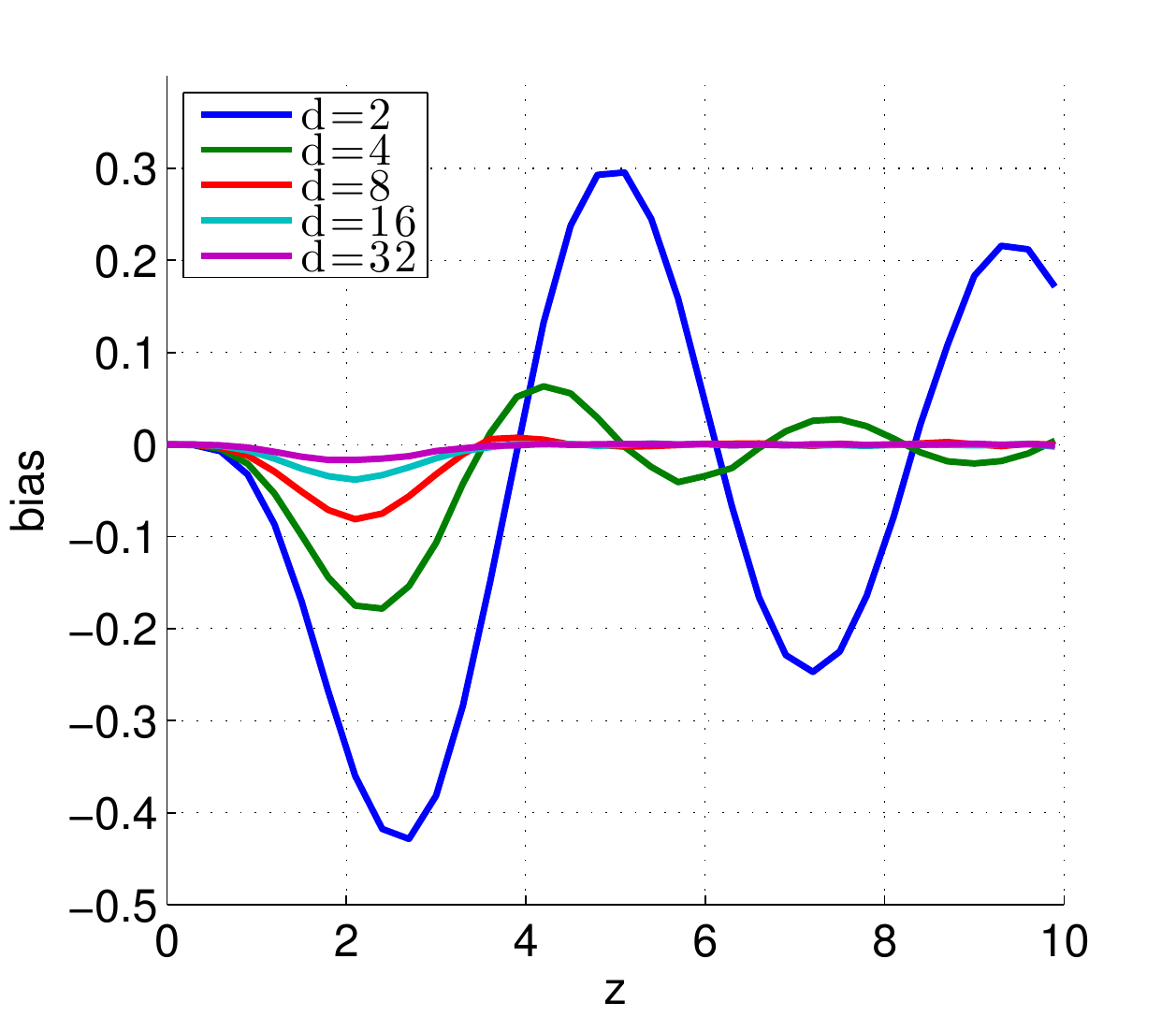}}
\hspace{-0.27cm}
\subfigure[Bias of SORF]{\includegraphics[width=0.26\textwidth]{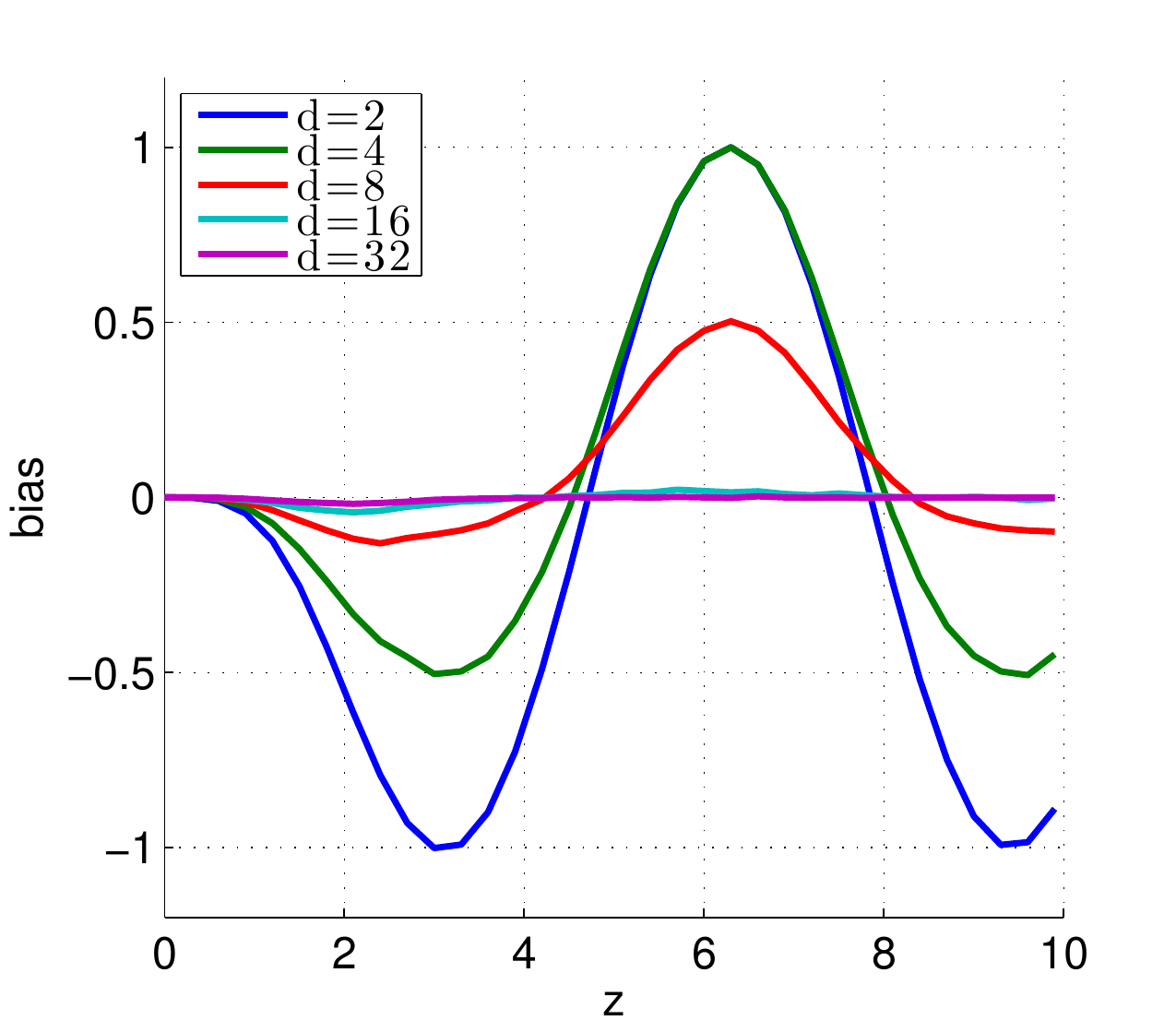}}
\hspace{-0.27cm}	
\subfigure[Variance ratio of ORF\textprime]{\includegraphics[width=0.26\textwidth]{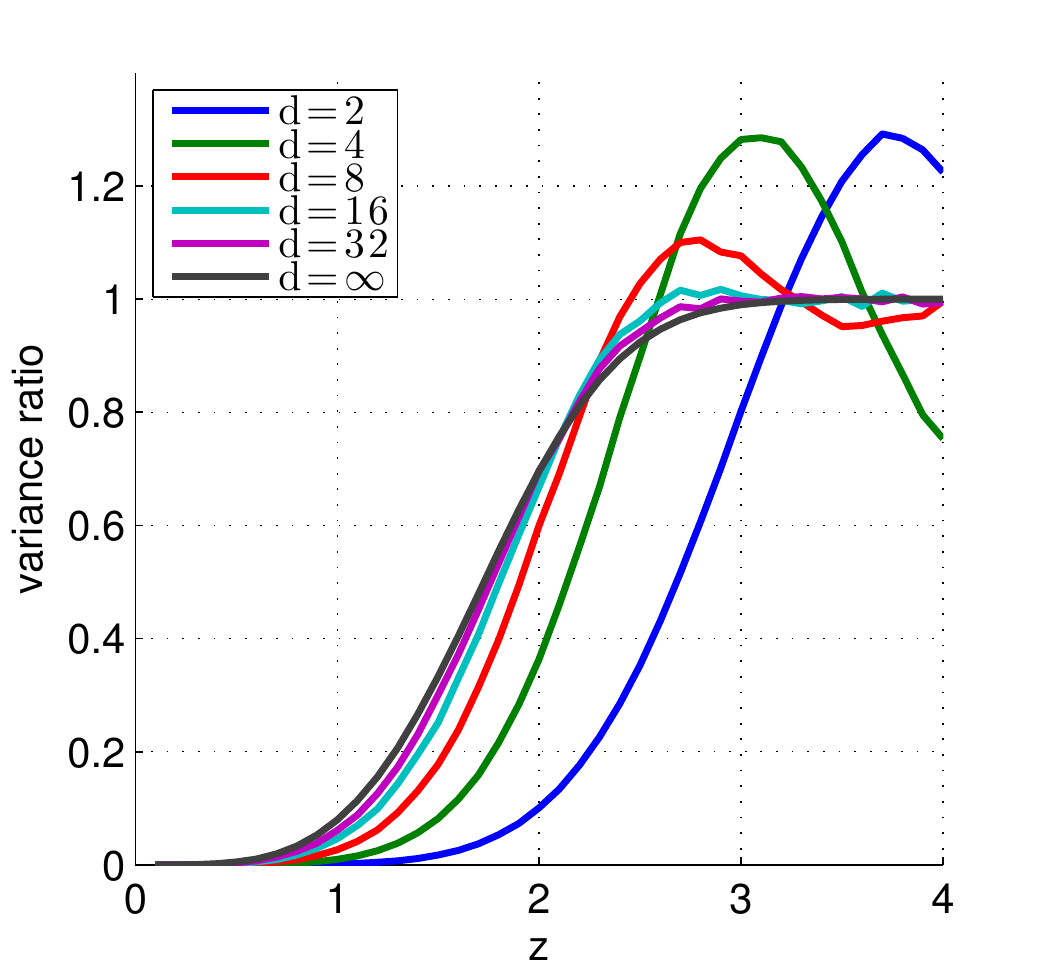}}
\hspace{-0.27cm}
\subfigure[Variance ratio of SORF]
{\includegraphics[width=0.26\textwidth]{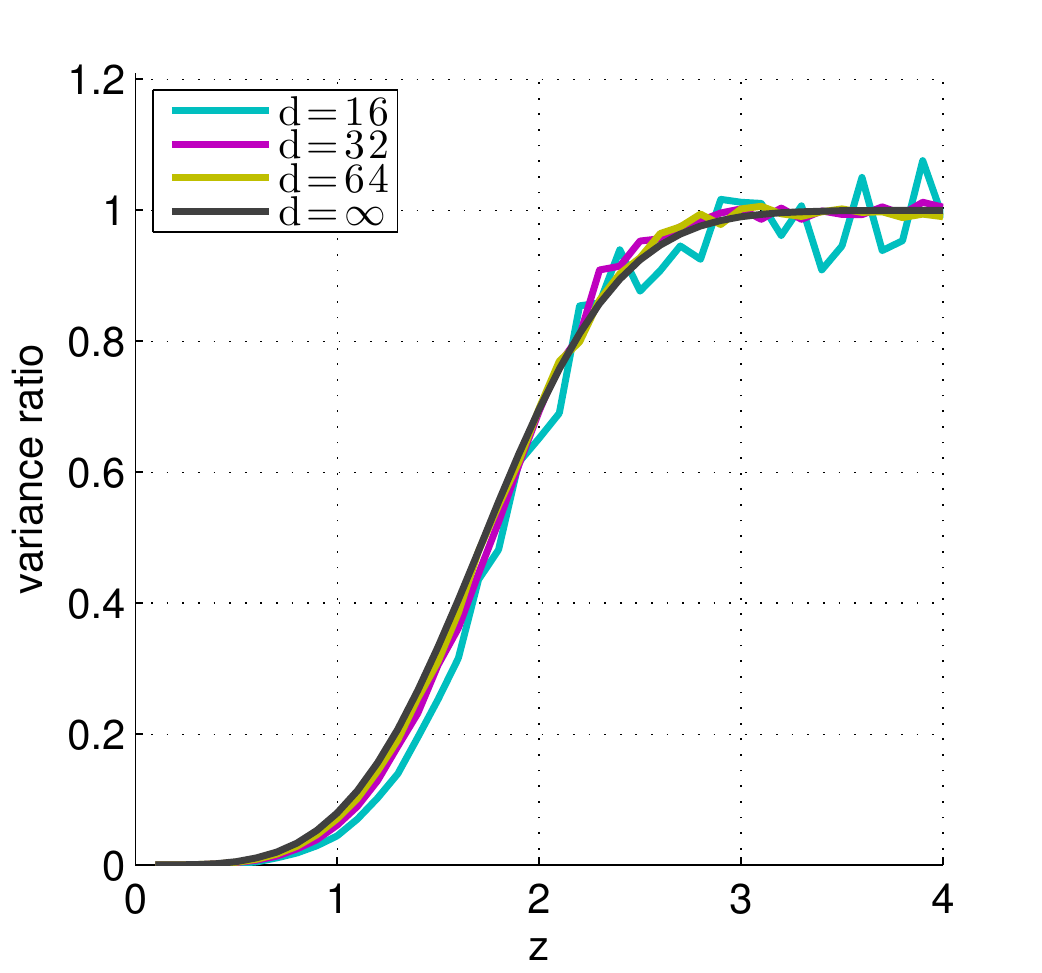}}
\hspace{-0.27cm}
\caption{Simulations of bias and variance of ORF\textprime and SORF. $z = ||\bx - \by|| / \sigma$. 
(a) $\mathbb{E}(K_{\ORF^\prime}(\bx, \by)) - e^{-z^2 / 2}$. 
(b) $\mathbb{E}(K_{\SORF}(\bx, \by)) - e^{-z^2 / 2}$. 
(c) $\Var(K_{\ORF^\prime}(\bx, \by)) / \Var(K_{\text{RFF}}(\bx, \by))$.
(d) $\Var(K_{\SORF}(\bx, \by)) / \Var(K_{\text{RFF}}(\bx, \by))$.
Each point on the curve is based on 20,000 choices of the random matrices and two fixed points with distance $z$. For both ORF and ORF\textprime, even at $d=32$, the bias is close to 0 and the variance is close to that of $d = \infty$ (Figure \ref{fig:variance}(a)).}
\label{fig:ORF_simple}
\vspace{-0.3cm}
\end{figure}

Recall that $z = ||\bx - \by|| / \sigma$. This means that ORF preserves the kernel value especially well for data points that are close, thereby retaining the local structure of the dataset.
Furthermore, empirically $\sigma$ is typically not set too small in order to prevent overfitting---a common rule of thumb is to set $\sigma$ to be the average distance of 50th-nearest neighbors in a dataset. In Figure \ref{fig:variance}(c), we plot the distribution of $z$ for several datasets with this choice of $\sigma$. These distributions are all concentrated in the regime where ORF yields substantial variance reduction.

The above analysis is under the assumption that $D \leq d$. Empirically, for RFF, $D$ needs to be larger than $d$ in order to achieve low approximation error. In that case, we independently generate and apply the transformation (\ref{eq:ORF}) multiple times. The next lemma bounds the variance for this case.
\begin{corollary}
Let $D = m \cdot d$, for an integer $m$ and $z = ||\bx - \by|| / \sigma$. There exists a function $f$ such that for all $z$, the variance of $K_{\text{ORF}}(\bx, \by)$ is bounded by
\[
\Var\left(K_{\text{ORF}}(\bx, \by) \right) \leq \frac{1}{2D} \left(\left(1-e^{-z^2}\right)^2 - \frac{d-1}{d}  e^{-z^2} z^4 \right) + \frac{f(z)}{dD}.
\]
\end{corollary}
\section{Structured Orthogonal Random Features}
\label{sec:SORF}

In the previous section, we presented Orthogonal Random Features (ORF) and provided a theoretical explanation for their effectiveness. Since generating orthogonal matrices in high dimensions can be expensive, here we propose a fast version of ORF by imposing structure on the orthogonal matrices. This method can provide drastic memory and time savings with minimal compromise on kernel approximation quality. Note that the previous works on fast kernel approximation using structured matrices do not use structured \textit{orthogonal} matrices \cite{le2013fastfood, arxiv_cnm, choromanski2016recycling}.

Let us first introduce a simplified version of ORF: replace $\mathbf{S}$ in (\ref{eq:ORF}) by a scalar $\sqrt{d}$. Let us call this method ORF\textprime. The transformation matrix thus has the following form:
\begin{equation}
	\bW_{\ORF^\prime} = \frac{\sqrt{d}}{\sigma} \mathbf{Q}.
	\label{eq:ORF_simple}
\end{equation}

\begin{theorem}(Appendix~\ref{app:orfd})
\label{thm:orfd}
Let $K_{\ORF^\prime}(\bx, \by)$ be the approximate kernel computed with linear transformation matrix (\ref{eq:ORF_simple}).
Let $D \leq d$ and $z = ||\bx - \by||/\sigma$. There exists a function $f$ such that the bias of $K_{\ORF^\prime}(\bx,\by)$ satisfies
\[
\vspace{-0.4cm}
\left \lvert \mathbb{E}(K_{\ORF^\prime}(\bx, \by)) -  e^{-z^2/2}\right \rvert \leq e^{-z^2/2} \frac{z^4}{4d} +  \frac{f(z)}{d^2},
\] 
and the variance satisfies
\[
\Var\left(K_{\ORF^\prime}(\bx, \by) \right) \leq \frac{1}{2D} \left((1-e^{-z^2})^2 - \frac{D-1}{d}  e^{-z^2} z^4 \right) + \frac{f(z)}{d^2}.
\]
\vspace{-0.2cm}
\end{theorem}

\vspace{-0.2cm}
The above implies that when $d$ is large $K_{\ORF^\prime}(\bx, \by)$ is a good estimation of the kernel with low variance.
Figure \ref{fig:ORF_simple}(a) shows that even for  relatively small $d$, the estimation is almost unbiased. Figure \ref{fig:ORF_simple}(c) shows that when $d \geq 32$, the variance ratio is very close to that of $d = \infty$. We find empirically that ORF\textprime also provides very similar MSE in comparison with ORF in real-world datasets. 

\begin{figure}[t]
\centering
\subfigure[\texttt{LETTER} ($d = 16$)]{\includegraphics[width=0.33\textwidth]{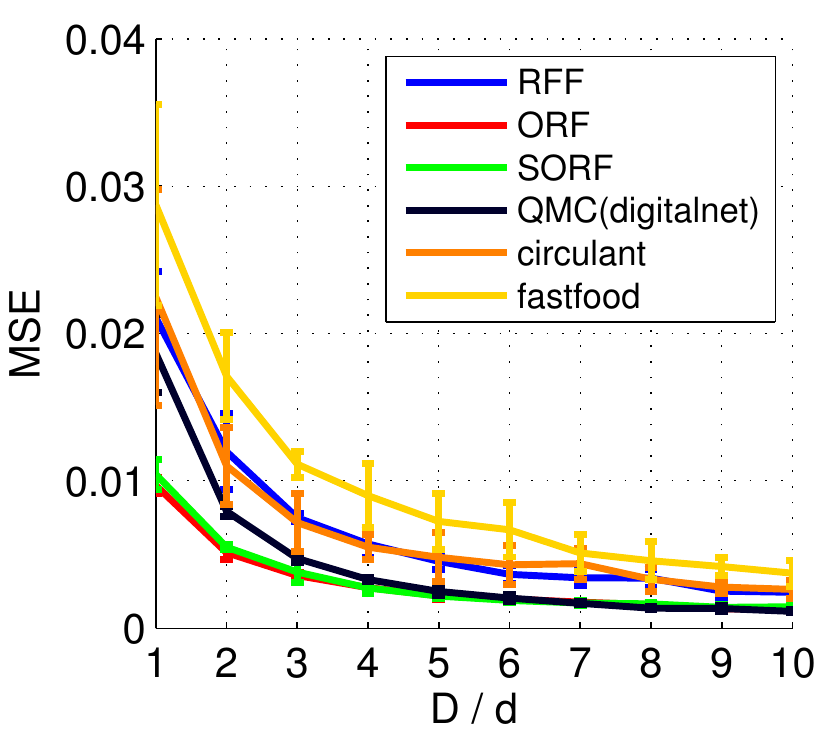}}
\hspace{-0.23cm}
\subfigure[\texttt{FOREST} ($d = 64$)]{\includegraphics[width=0.33\textwidth]{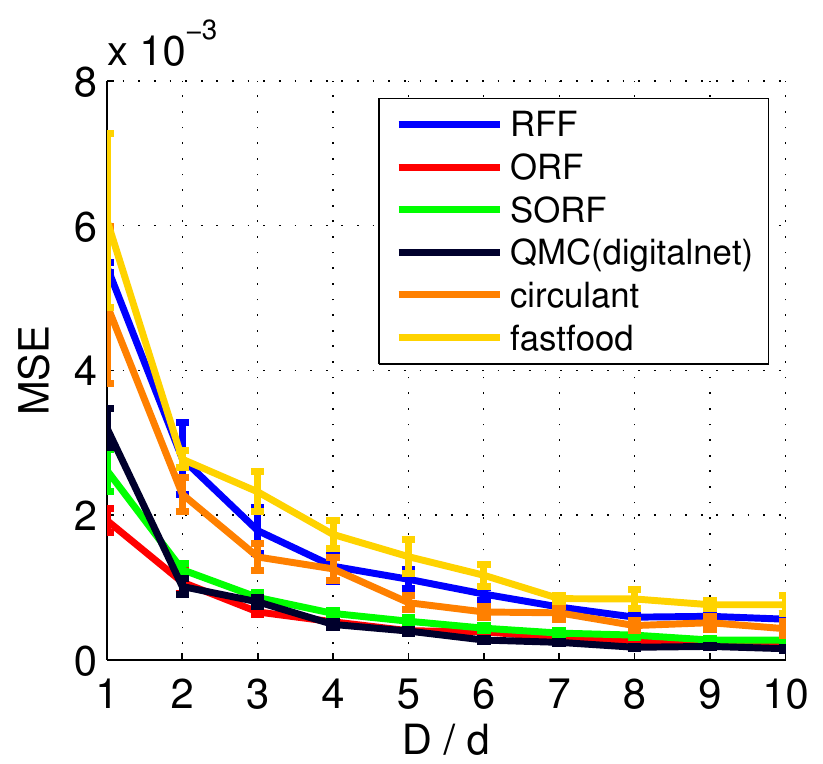}}
\hspace{-0.23cm}
\subfigure[\texttt{USPS} ($d = 256$)]{\includegraphics[width=0.33\textwidth]{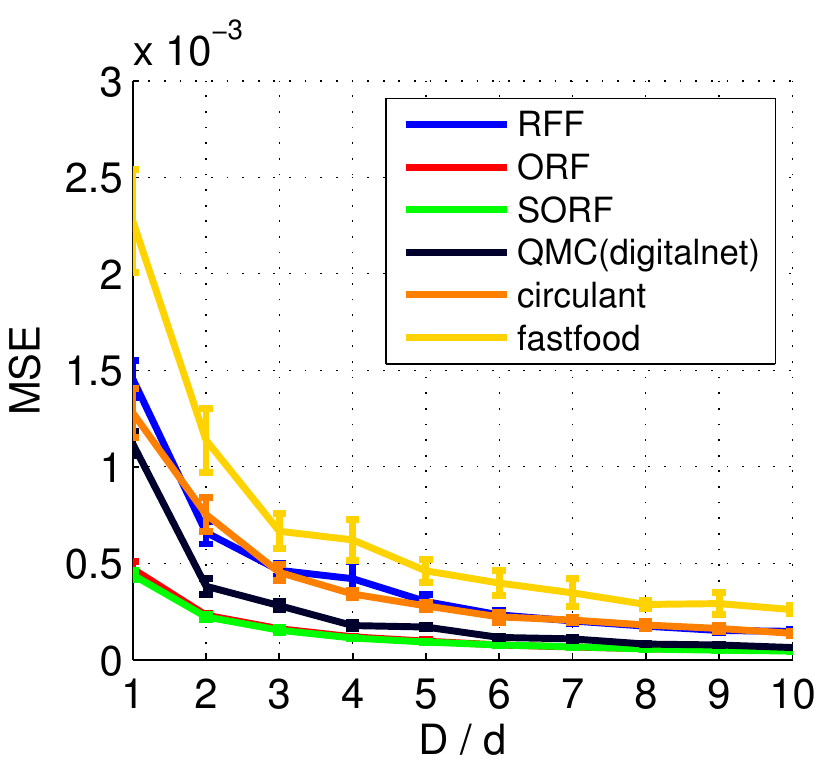}}
\subfigure[\texttt{CIFAR} ($d = 512$)]{\includegraphics[width=0.33\textwidth]{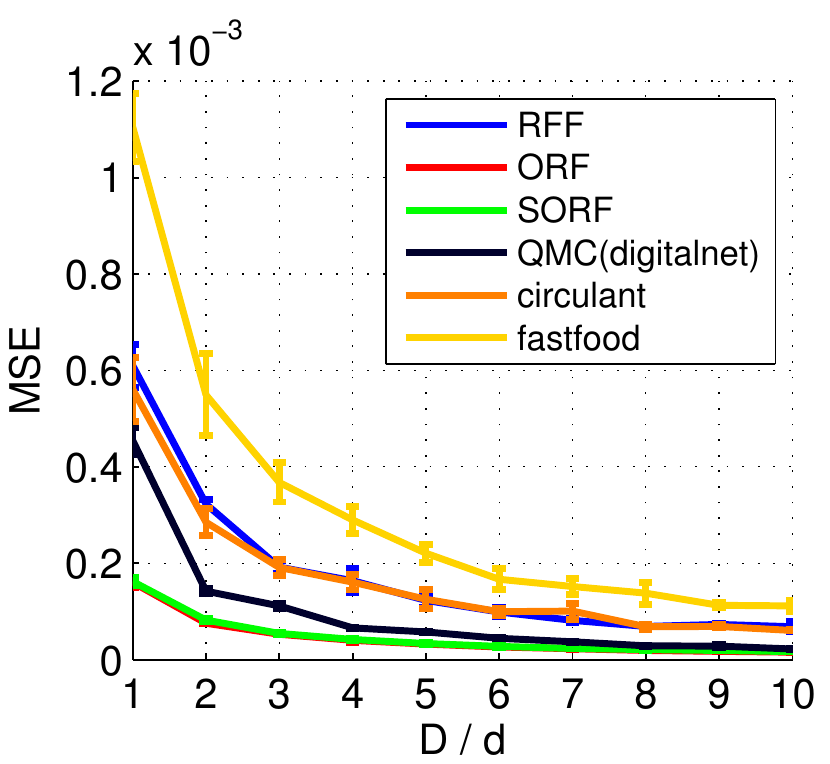}}
\hspace{-0.23cm}
\subfigure[\texttt{MNIST} ($d = 1024$)]{\includegraphics[width=0.33\textwidth]{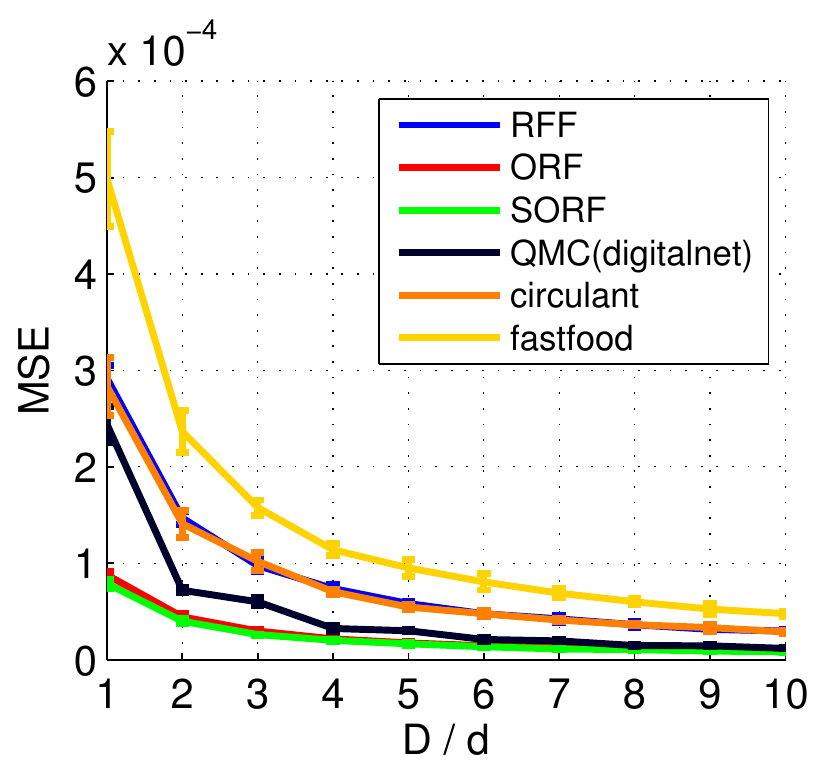}}
\hspace{-0.23cm}
\subfigure[\texttt{GISETTE} ($d = 4096$)]{\includegraphics[width=0.33\textwidth]{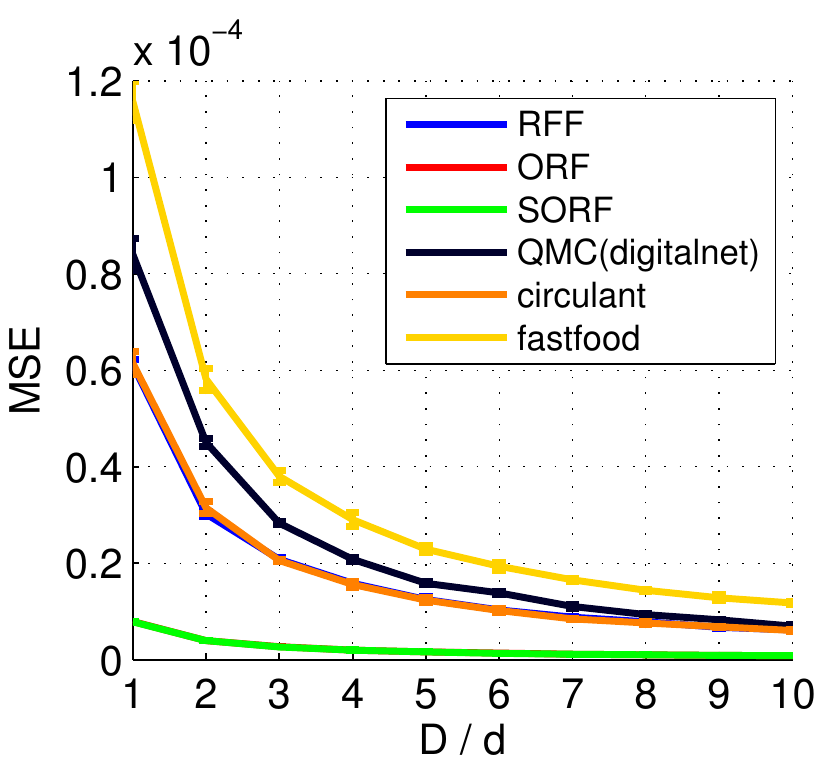}}
\vspace{-0.2cm}
\caption{Kernel approximation mean squared error (MSE) for the Gaussian kernel $K(\bx, \by) = e^{-||\bx - \by ||^2 / 2\sigma^2}$. $D$: number of transformations. $d$: input feature dimension.  For each dataset, $\sigma$ is chosen to be the mean distance of the 50th $\ell_2$ nearest neighbor for 1,000 sampled datapoints. Empirically, this yields good classification results. The curves for SORF and ORF overlap.}
\label{fig:mse_all}
\vspace{-0.4cm}
\end{figure}

We now introduce Structured Orthogonal Random Features (SORF). It replaces the random orthogonal matrix $\bQ$ of ORF\textprime in (\ref{eq:ORF_simple}) by a special type of structured matrix $\bH \bD_1 \bH \bD_2 \bH \bD_3$:
\vspace{-0.2cm}
\begin{equation}
	\bW_{\text{SORF}} = \frac{\sqrt{d}}{\sigma} \bH \bD_1 \bH \bD_2 \bH \bD_3,
	\label{eq:SORF}
\end{equation}

\vspace{-0.4cm}
where $\bD_i \in \mathbb{R}^{d \times d}, i = 1, 2, 3$ are diagonal ``sign-flipping'' matrices, with each diagonal entry sampled from the Rademacher distribution. $\bH$ is the normalized Walsh-Hadamard matrix.

Computing $\bW_{\text{SORF}} \bx$ has the time cost $\mathcal{O}(d \log d)$, since multiplication with $\bD$ takes $\mathcal{O}(d)$ time and multiplication with $\bH$ takes $\mathcal{O}(d \log d)$ time using fast Hadamard transformation. 
The computation of SORF can also be carried out with almost no extra memory due to the fact that both sign flipping and the Walsh-Hadamard transformation can be efficiently implemented as in-place operations \cite{fino1976unified}.

Figures \ref{fig:ORF_simple}(b)(d) show the bias and variance of SORF. Note that although the curves for small $d$ are different from those of ORF, when $d$ is large ($d > 32$ in practice), the kernel estimation is almost unbiased, and the variance ratio converges to that of ORF.  
In other words, it is clear that SORF can provide almost identical kernel approximation quality as that of ORF. This is also confirmed by the experiments in Section \ref{sec:exp}. In Section \ref{sec:extensions}, we provide theoretical discussions to show  that the structure of  (\ref{eq:SORF}) can also be generally applied to many scenarios where random Gaussian matrices are used.

\section{Experiments}
\label{sec:exp}

\textbf{Kernel Approximation.}
We first show kernel approximation performance on six datasets. The input feature dimension $d$ is set to be power of 2 by padding zeros or subsampling. Figure \ref{fig:mse_all} compares the mean squared error (MSE) of all methods. For fixed $D$, the kernel approximation MSE exhibits the following ordering:
\begin{equation*}
\text{SORF $\simeq$ ORF $<$ QMC \cite{yang2014quasi} $<$ RFF \cite{rahimi2007random} $<$ Other fast kernel approximations \cite{le2013fastfood, arxiv_cnm}}.
\end{equation*}
By imposing orthogonality on the linear transformation matrix, Orthogonal Random Features (ORF) achieves significantly lower approximation error than Random Fourier Features (RFF). The Structured Orthogonal Random Features (SORF) have almost identical MSE to that of ORF. All other fast kernel approximation methods, such as circulant \cite{arxiv_cnm} and FastFood \cite{le2013fastfood} have higher MSE. We also include DigitalNet, the best performing method among Quasi-Monte Carlo techniques \cite{yang2014quasi}. Its MSE is lower than that of RFF, but still higher than that of ORF and SORF. The order of time cost for a fixed $D$ is
\begin{equation*}
\text{SORF $\simeq$ Other fast kernel approximations \cite{le2013fastfood, arxiv_cnm} $\ll$ ORF $=$ QMC \cite{yang2014quasi} $=$ RFF \cite{rahimi2007random}}.
\end{equation*}
Remarkably, SORF has  both better computational efficiency and higher kernel approximation quality
compared to other methods.

\begin{table}[t]
\vspace{-0.1cm}
{\small
\hfill{}
\scalebox{0.9}{
\begin{tabular}{l||l||c|c|c|c|c||c}
\hline
\textbf{Dataset}&\textbf{Method}& $D = 2d$ &  $D = 4d$ &  $D = 6d$ &  $D = 8d$ &  $D = 10d$ &  Exact\\
\hline
\hline
\multirow{3}{*}{\parbox{1cm}{\texttt{letter} $d$ = 16}}
& RFF  & 76.44 $\pm$ 1.04 & 81.61 $\pm$ 0.46 &  \textbf{85.46 $\pm$ 0.56} & 86.58 $\pm$ 0.99 &  \textbf{87.84 $\pm$ 0.59} & \multirow{3}{*}{\parbox{1cm}{\texttt{90.10}}}\\
& ORF  &  \textbf{77.49 $\pm$ 0.95} &  \textbf{82.49 $\pm$ 1.16} & 85.41 $\pm$ 0.60 &  \textbf{87.17 $\pm$ 0.40} & 87.73 $\pm$ 0.63 \\
& SORF  & 76.18 $\pm$ 1.20 & 81.63 $\pm$ 0.77 & 84.43 $\pm$ 0.92 & 85.71 $\pm$ 0.52 & 86.78 $\pm$ 0.53 \\
\hline

\multirow{3}{*}{\parbox{1cm}{\texttt{forest} $d$ = 64}}
& RFF  & 77.61 $\pm$ 0.23 &  \textbf{78.92 $\pm$ 0.30} & 79.29 $\pm$ 0.24 & 79.57 $\pm$ 0.21 &  \textbf{79.85 $\pm$ 0.10} & \multirow{3}{*}{\parbox{1cm}{\texttt{80.43}}}\\
& ORF  &  \textbf{77.88 $\pm$ 0.24} & 78.71 $\pm$ 0.19 &  \textbf{79.38 $\pm$ 0.19} &  \textbf{79.63 $\pm$ 0.21} & 79.54 $\pm$ 0.15 \\
& SORF  & 77.64 $\pm$ 0.20 & 78.88 $\pm$ 0.14 & 79.31 $\pm$ 0.12 & 79.50 $\pm$ 0.14 & 79.56 $\pm$ 0.09 \\
\hline

\multirow{3}{*}{\parbox{1cm}{\texttt{usps} \\ $d$ = 256}}
& RFF  & 94.27 $\pm$ 0.38 & 94.98 $\pm$ 0.10 & 95.43 $\pm$ 0.22 &  \textbf{95.66 $\pm$ 0.25} & 95.71 $\pm$ 0.18 & \multirow{3}{*}{\parbox{1cm}{\texttt{95.57}}}\\
& ORF  & 94.21 $\pm$ 0.51 &  \textbf{95.26 $\pm$ 0.25} &  \textbf{96.46 $\pm$ 0.18} & 95.52 $\pm$ 0.20 &  \textbf{95.76 $\pm$ 0.17} \\
& SORF  &  \textbf{94.45 $\pm$ 0.39} & 95.20 $\pm$ 0.43 & 95.51 $\pm$ 0.34 & 95.46 $\pm$ 0.34 & 95.67 $\pm$ 0.15 \\
\hline

\multirow{3}{*}{\parbox{1cm}{\texttt{cifar} \\ $d$ = 512}}
& RFF  & 73.19 $\pm$ 0.23 & 75.06 $\pm$ 0.33 & 75.85 $\pm$ 0.30 & 76.28 $\pm$ 0.30 & 76.54 $\pm$ 0.31 & \multirow{3}{*}{\parbox{1cm}{\texttt{78.71}}}\\
& ORF  &  \textbf{73.59 $\pm$ 0.44} & 75.06 $\pm$ 0.28 &  \textbf{76.00 $\pm$ 0.26} & 76.29 $\pm$ 0.26 &  \textbf{76.69 $\pm$ 0.09} \\
& SORF  & 73.54 $\pm$ 0.26 &  \textbf{75.11 $\pm$ 0.21} & 75.76 $\pm$ 0.21 &  \textbf{76.48 $\pm$ 0.24} & 76.47 $\pm$ 0.28 \\
\hline

\multirow{3}{*}{\parbox{1.1cm}{\texttt{mnist} \\ $d$ = 1024}}
& RFF  & 94.83 $\pm$ 0.13 & 95.48 $\pm$ 0.10 & 95.85 $\pm$ 0.07 &  \textbf{96.02 $\pm$ 0.06} & 95.98 $\pm$ 0.05 & \multirow{3}{*}{\parbox{1cm}{\texttt{97.14}}}\\
& ORF  & 94.95 $\pm$ 0.25 &  \textbf{95.64 $\pm$ 0.06} &  \textbf{95.85 $\pm$ 0.09} & 95.95 $\pm$ 0.08 &  \textbf{96.06 $\pm$ 0.07} \\
& SORF  &  \textbf{94.98 $\pm$ 0.18} & 95.48 $\pm$ 0.08 & 95.77 $\pm$ 0.09 & 95.98 $\pm$ 0.05 & 96.02 $\pm$ 0.07 \\
\hline

\multirow{3}{*}{\parbox{1.2cm}{\texttt{gisette} \\ $d$ = 4096}}
& RFF  &  \textbf{97.68 $\pm$ 0.28} &  \textbf{97.74 $\pm$ 0.11} & 97.66 $\pm$ 0.25 &  \textbf{97.70 $\pm$ 0.16} &  \textbf{97.74 $\pm$ 0.05} & \multirow{3}{*}{\parbox{1cm}{\texttt{97.60}}}\\
& ORF  & 97.56 $\pm$ 0.17 & 97.72 $\pm$ 0.15 &  \textbf{97.80 $\pm$ 0.07} & 97.64 $\pm$ 0.09 & 97.68 $\pm$ 0.04 \\
& SORF  & 97.64 $\pm$ 0.17 & 97.62 $\pm$ 0.04 & 97.64 $\pm$ 0.11 & 97.68 $\pm$ 0.08 & 97.70 $\pm$ 0.14 \\
\hline
\hline
\end{tabular}}}
\hfill{}
\caption{Classification Accuracy based on SVM. ORF and SORF provide competitive classification accuracy for a given $D$. Exact is based on kernel-SVM trained on the Gaussian kernel. 
Note that in all the settings SORF is faster than RFF and ORF by a factor of $\mathcal{O}(d / \log d)$. For example, on \texttt{gisette} with $D = 2d$, SORF provides 10 times speedup in comparison with RFF and ORF.}
\label{table:acc}
\vspace{-0.4cm}
\end{table}

We also apply ORF and SORF on classification tasks. Table \ref{table:acc} shows classification accuracy for different kernel approximation techniques with a (linear) SVM classifier. SORF is competitive with or better than RFF, and has greatly reduced time and space costs.

\textbf{The Role of $\sigma$.}
Note that a very small $\sigma$ will lead to overfitting, and a very large $\sigma$ provides no discriminative power for classification. 
Throughout the experiments, $\sigma$ for each dataset is chosen to be the mean distance of the 50th $\ell_2$ nearest neighbor, which empirically yields good classification results \cite{arxiv_cnm}. 
As shown in Section \ref{sec:ORF}, the relative improvement over RFF is positively correlated with $\sigma$. Figure \ref{fig:sigma}(a)(b) verify this on the \texttt{mnist} dataset. Notice that the proposed methods (ORF and SORF) consistently improve over RFF. 

\textbf{Simplifying SORF.}
The SORF transformation consists of three Hadamard-Diagonal blocks. A natural question is whether using fewer computations and randomness can achieve similar empirical performance. 
Figure \ref{fig:sigma}(c) shows that reducing the number of  blocks to two (HDHD) provides similar performance, while reducing to one block (HD) leads to large error. 

\begin{figure}[t]
\centering\subfigure[$\sigma = 0.5 \times$ 50NN distance]{\includegraphics[width=0.33\textwidth]{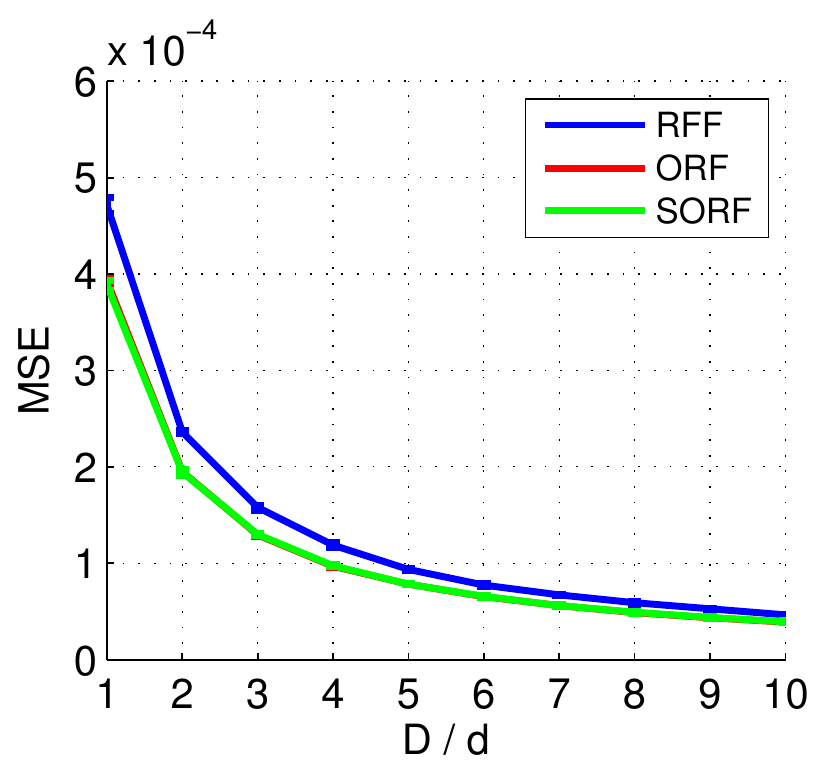}}
\hspace{-0.23cm}
\subfigure[$\sigma =2 \times$ 50NN distance]{\includegraphics[width=0.33\textwidth]{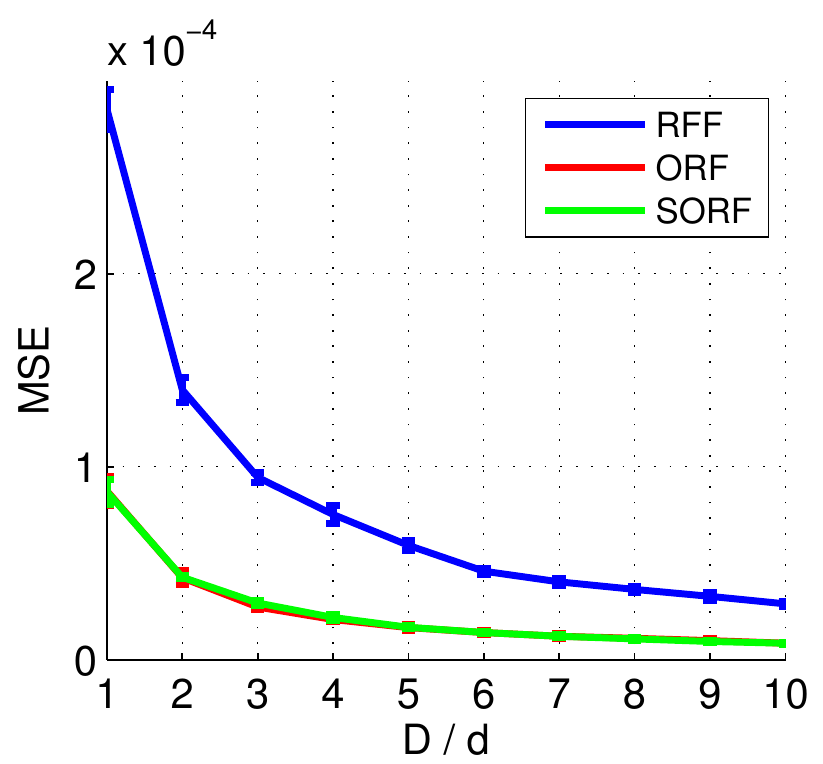}}
\hspace{-0.23cm}
\subfigure[Variants of SORF]{\includegraphics[width=0.33\textwidth]{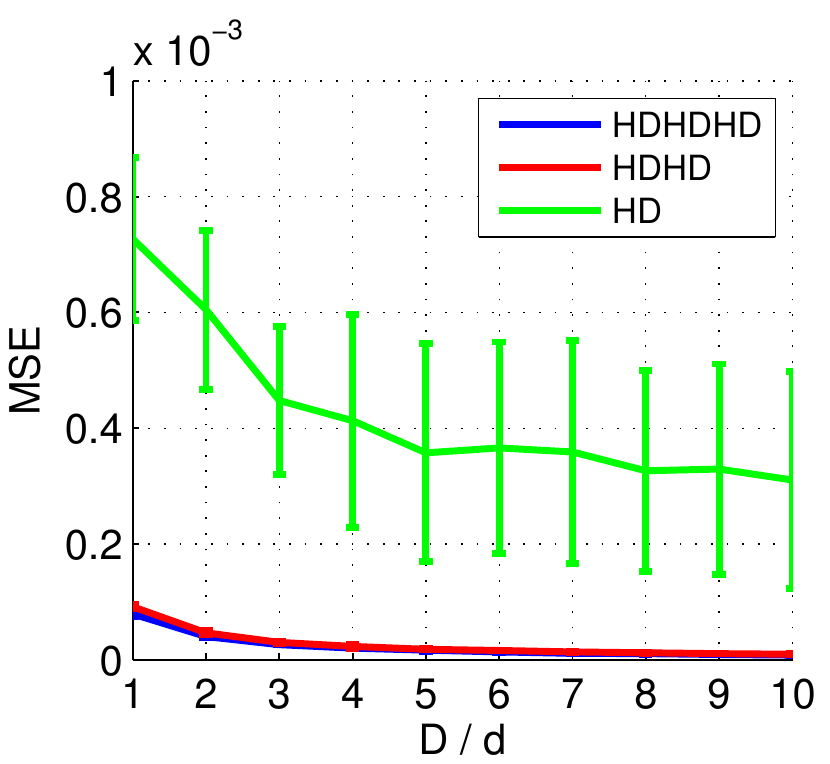}}
\caption{(a) (b) MSE on \texttt{mnist} with different $\sigma$. (c) Effect of using less randomness on \texttt{mnist}. HDHDHD is the the proposed SORF method. HDHD reduces the number of Hadamard-Diagonal blocks to two, and HD uses only one such block.}
\vspace{-0.3cm}
\label{fig:sigma}
\end{figure}

\section{Analysis and General Applicability of the Hadamard-Diagonal Structure}
\label{sec:extensions}

We provide theoretical discussions of SORF in this section. 
We first show that for large $d$, SORF is an unbiased estimator of the Gaussian kernel. 
\begin{theorem} (Appendix \ref{app:thm:biassorf})
\label{thm:biassorf}
Let $K_{\SORF}(\bx, \by)$ be the approximate kernel computed with linear transformation matrix $\sqrt{d} \bH \bD_1 \bH \bD_2 \bH \bD_3$.
Let $z = ||\bx - \by||/\sigma$. Then 
\[
\left \lvert \mathbb{E}(K_{\SORF}(\bx, \by)) -  e^{-z^2/2}\right \rvert \leq  \frac{6z}{\sqrt{d}}.
\] 
\end{theorem}

Even though SORF is nearly-unbiased, proving tight variance and concentration guarantees similar to ORF remains an open question. 
The following discussion provides a sketch in that direction.
We first show a lemma of RFF.
\begin{lemma}
Let $\bW$ be a random Gaussian matrix as in RFF, for a given $\bz$, the distribution of $\bW \bz$ is $N(0, ||z||_2 \mathbf{I}_d)$.
\end{lemma}
Note that $\bW \bz$ in RFF can be written as $\bR \bg$, where $\bR$ is a scaled orthogonal matrix such that each row has norm $||z||_2$ and $\bg$ is distributed according to $N(0, \mathbf{I}_d)$. Hence the distribution of $\bR \bg$ is $N(0, ||z||_2 \mathbf{I}_d)$, identical to $\bW \bz$. The concentration results of RFF use the fact that the projections of a Gaussian vector $\bg$ onto orthogonal directions $\bR$ are independent.

We show that $\sqrt{d} \bH \bD_1 \bH \bD_2 \bH \bD_3 \bz $ has similar properties. In particular, we show that it 
can be written as $\tilde\bR \tilde\bg$, where  rows of $\tilde \bR$ are ``near-orthogonal'' (with high probability) and have norm $||\bz ||_2$, and the vector $\tilde \bg$ is close to Gaussian ($\tilde \bg$ has independent sub-Gaussian elements), and hence the projections behave ``near-independently''. 
Specifically, $ \tilde \bg = \text{vec}(\bD_1)$ (vector of diagonal entries of $\bD_1$), and $\tilde{\bR}$ is a function of $\bD_2$, $\bD_3$ and $\bz$.  
\begin{theorem} (Appendix \ref{app:thm:nearortho})
\label{thm:nearortho}
For a given $\bz$, there exists a $\tilde \bR$ (function of $\bD_2, \bD_3, \bz$), such that $\sqrt{d} \bH \bD_1 \bH \bD_2 \bH \bD_3 \bz  = \tilde \bR \text{vec}(\bD_1)$. Each row of $\tilde \bR$ has norm $||\bz||_2$ and for any $t \geq 1/d$, with probability $1-de^{-c\cdot t^{2/3} d^{1/3}}$, the inner product between any two rows of $\tilde\bR$ is at most $t ||z||_2$, where $c$ is a constant.
\end{theorem}
The above result can also be applied to settings not limited to kernel approximation. In the appendix, we show empirically that the same scheme can be successfully applied to angle estimation where the nonlinear map $f$ is a non-smooth $\sign(\cdot)$ function \cite{charikar2002similarity}. We note that the $\bH\bD_1\bH\bD_2\bH\bD_3$ structure has also been recently used in fast cross-polytope LSH \cite{andoni2015practical, kennedy2016fast,choromanski2016triplespin}. 

\vspace{-0.2cm}
\section{Conclusions}
\vspace{-0.2cm}
We have demonstrated that imposing orthogonality on the transformation matrix can greatly reduce the kernel approximation MSE of Random Fourier Features when approximating Gaussian kernels.
We further proposed a type of structured orthogonal matrices with substantially lower computation and memory cost. 
We provided theoretical insights indicating that the Hadamard-Diagonal block structure can be generally used to replace random Gaussian matrices in a broader range of applications. Our method can also be generalized to other types of kernels such as general shift-invariant kernels and polynomial kernels based on Schoenberg's characterization as in \cite{pennington2015spherical}.

\bibliographystyle{abbrv}
\bibliography{thesis}

\begin{appendices}

\section{Variance Reduction via Orthogonal Random Features}
\subsection{Notation}
Let $\bz = \frac{\bx-\by}{\sigma}$, and $z = ||\bz||$. For a vector $\by$, let $y(i)$ denote its $i^{\text{th}}$ coordinate. Let $n!!$ be the double factorial of $n$, i.e., the product of every number from n to 1 that has the same parity as n. 
\subsection{Proof of Lemma~\ref{lem:RFF}}
\label{app:RFF}
Let $\bz  = (\bx- \by)/\sigma$. 
Recall that in RFF, we compute the Kernel approximation as 
\[
\sum^D_{i=1} \frac{1}{D} \cos(\bw^T_i \bz),
\]
where each $\bw_i$ is a $d$ dimensional vector distributed $N(0, I_d)$.
Let $\bw$ be a $d$ dimensional vector distributed $N(0,I_d)$. 
By Bochner's theorem,
\[
\E[\cos(\bw^T \bz)] = e^{-z^2/2},
\]
and hence RFF yields an unbiased estimate. 

We now compute the variance of RFF approximation.
Observe that 
\[
\cos^2(\bw^T \bz) = \frac{1 + \cos(2\bw^T \bz)}{2} = \frac{1 + \cos(\bw^T (2\bz))}{2}  .
\]
Hence by Bochner's theorem
\[
\E[\cos^2 (\bw^T \bz)] = \frac{1 + e^{-2z^2}}{2}.
\]
Therefore,
\begin{align*}
\Var(\cos(\bw^T \bz))
&= \E[\cos^2 (\bw^T \bz)] - (\E[\cos (\bw^T \bz)] )^2  \\
& =  \frac{1 + e^{-2z^2}}{2} - e^{-z^2} = \frac{(1-e^{-z^2})^2}{2}.
\end{align*}
If we take $D$ such independent random variables $\bw_1, \bw_2, \ldots \bw_D$, since variance of the sum
is sum of variances,
\[
\Var\left(\frac{1}{D} \sum^D_{i=1} \cos(\bw^T_i \bz) \right) = \frac{(1-e^{-z^2})^2}{2D}.
\]
\newcommand{\ignore}[1]{}
\subsection{Proof of Lemma~\ref{lem:neg_cor}}
\label{app:neg_cor}
The proof uses the following lemma.
\begin{lemma}
\label{lem:tech1}
For a set of non-negative values $\alpha_1,\alpha_2,\ldots \alpha_k$ and $\beta_1,\beta_2,\ldots \beta_k$ such that for all $i$, $\beta_i \leq \alpha_i$,
\[
\left \lvert \frac{1}{\prod^k_{i} (1+ \alpha_i)}  - \left( 1- \sum^k_{i=1} \alpha_i\right) \right \rvert \leq \left(\sum^k_{i=1} \alpha_i\right)^2,
\]
and
\[
\left \lvert\prod^k_{i}  \frac{1+ \beta_i}{ 1+ \alpha_i}  - \left( 1 + \sum^k_{i=1} \beta_i - \sum^k_{i=1} \alpha_i \right) \right \rvert \leq  \left(\sum^k_{i=1} (\alpha_i - \beta_i) \right)^2 + \sum^k_{i=1} (\alpha_i -\beta_i)\beta_i.
\]
\end{lemma}
\begin{proof}
Since $\alpha_i$s are non-negative,
\begin{align*}
\frac{1}{\prod^k_{i} (1+ \alpha_i)} - \bigl( 1- \sum^k_{i=1} \alpha_i \bigr) 
& \leq  \frac{1}{1 + \sum^k_{i} \alpha_i} - \bigl( 1- \sum^k_{i=1} \alpha_i\bigr) \\
& = \frac{1 - ( 1 + \sum^k_{i=1} \alpha_i) ( 1- \sum^k_{i=1} \alpha_i) }{1 + \sum^k_{i} \alpha_i}  \\
& = \frac{(\sum^k_{i=1} \alpha_i)^2}{1 + \sum^k_{i} \alpha_i}  \leq \bigl(\sum^k_{i=1} \alpha_i \bigr)^2.
\end{align*}
Furthermore, by convexity
\begin{align*}
\frac{1}{\prod^k_{i} (1+ \alpha_i)}  
\geq \frac{1}{(1  + \sum^k_{i=1} \alpha_i /k)^k} 
\geq  e^{- \sum^k_{i=1} \alpha_i} 
\geq 1 - \sum^k_{i=1} \alpha_i.
\end{align*}
Combining the above two equations results in the first part of the lemma. For the second part observe that 
\[
\prod^k_{i}  \frac{1+ \beta_i}{ 1+ \alpha_i} = \frac{1}{\prod^k_{i} \left(1+ \frac{\alpha_i-\beta_i}{1+\beta_i}\right)}.
\]
Hence, by the first part
\[
\left \lvert \prod^k_{i}  \frac{1+ \beta_i}{ 1+ \alpha_i} - \left( 1- \sum^k_{i=1} \frac{\alpha_i - \beta_i}{1+\beta_i}\right) \right \rvert 
\leq \left(\sum^k_{i=1} \frac{\alpha_i - \beta_i}{1+\beta_i} \right)^2 \leq  \left(\sum^k_{i=1} (\alpha_i - \beta_i) \right)^2.
\]
Furthermore, for every $i$
\[
\left \lvert \frac{1}{1+\beta_i}  - 1 \right\rvert \leq \beta_i.
\]
Combining the above two equations yields the second part of the lemma.
\end{proof}
\ignore{
Let $\bw_1, \bw_2, \dots,  \bw_D$ be $D$ random variables that are designed as follows.
Let $\bu_1, \bu_2, \ldots, \bu_d$ be a uniformly chosen random basis of $\bR^d$.
For $1 \leq i \leq D$,
\[
\bw_i = s_i \bu_i,
\]
where $s_i$s are independent $\chi$-distributed random variables with $d$ degrees of freedom. It is easy to show that for each $i$, $\bw_i$ is distributed according to $N(0, I_d)$, and hence as
before 
\[
\E[\cos(\bw^T \bz)] = e^{-z^2/2}.
\]
However, we show that the bias is significantly smaller.
Let $a_i = \cos(\bw^T_i \bz)$.
Computing as before,
\begin{align*}
\Var\left(\frac{1}{D}\sum^D_{i=1} a_i\right) 
& = \E\left[\left( \frac{\sum^D_{i=1} a_i}{D} \right)^2\right] - \E\left[\left( \frac{\sum^D_{i=1} a_i}{D} \right)\right]^2 \\
& = \frac{1}{D^2}\sum_i \left(\E[a^2_i]  - \E[a_i]^2\right) +  \frac{1}{D^2}\sum_i \sum_{j \neq i} \left(\E[a_i a_j]  - \E[a_i]\E[a_j]\right) \\
& = \frac{(1-e^{-z^2})^2}{2D} + \frac{D(D-1)}{D^2} \left(\E[a_1 a_2]  - e^{-z^2}\right),
\end{align*}
where the last equality follows from symmetry and the variance calculations in the previous case.

We now show that $\E[a_1 a_2] < e^{-z^2}$, asymptotically, and hence we benefit from the newer scheme.
\begin{lemma}
For every $z$ there exist a $d_z$ such that for all $d > d_z$,
\[
\E[\cos(\bw^T_1 \bz) \cos(\bw^T_2 \bz)] \leq e^{-z^2} - e^{-z^2}\frac{z^4}{2d} + o(\frac{1}{d}). 
\]
\end{lemma}
}
\begin{proof}[Proof of Lemma~\ref{lem:neg_cor}]
Observe that 
\[
\cos(\bw^T_1 \bz) \cos(\bw^T_2 \bz)  = \frac{cos(\bw^T_1 \bz + \bw^T_2 \bz) +  \cos(\bw^T_1 \bz - \bw^T_2 \bz)}{2}.
\]
Since the problem is rotation invariant, instead of projecting a vector $\bz$ onto a randomly chosen two orthogonal vectors $\bu_1$ and $\bu_2$, we can choose a vector $\by$ that is uniformly distributed on a sphere of radius $z$ and project it on to the first two dimensions. 
Thus,
\[
\E[\cos(\bw^T_1 \bz + \bw^T_2 \bz) ]= \E[\cos((s_1 y(1) + s_2 y(2))z)].
\]
Similarly, 
\[
\E[\cos(\bw^T_1 \bz - \bw^T_2 \bz) ]= \E[\cos((s_1 y(1) - s_2 y(2))z)].
\]
The $k^{\text{th}}$ term in the Taylor's series expansion of sum of above two terms is 
\begin{align*}
& \frac{(-1)^k}{(2k)!} \left( (s_1y(1) + s_2 y(2))z\right)^{2k}
+ \frac{(-1)^k}{(2k)!} \left( (s_1y(1) - s_2 y(2))z\right)^{2k} \\
& =  \frac{(-z^2)^k}{(2k)!} 
\sum^{k}_{i=0}  {2k \choose 2i} s^{2i}_1 y^{2i}(1) s^{2k-2i}_2 y^{2k-2i}(2).
\end{align*}
A way to compute a uniformly distributed random variable on a sphere with radius $z$ is to generate
$d$ independent random variables $\bx = (x(1),x(2),\ldots, x(d) )$ each distributed $N(0,1)$ and
setting $y(i) = z x(i)/||\bx||$. Hence,
\begin{align*}
& \E \left[\sum^{k}_{i=0}  {2k \choose 2i} s^{2i}_1 y^{2i}(1) s^{2k-2i}_2 y^{2k-2i}(2)\right] \\
& \stackrel{(a)}{=} \E \left[\sum^{2k}_{i=0}   {2k \choose 2i} \frac{s^{2i}_1 x^{2i}(1) s^{2k-2i}_2 x^{2k-2i}(2)}{ ||\bx||^{2k}}\right] \\
& \stackrel{(b)}{=} \sum^{k}_{i=0}  {2k \choose 2i} \E[s^{2i}_1]  \E[s^{2k- 2i}_2] \E \left[ \frac{x^{2i}(1) x^{2k-2i}(2)}{||\bx||^{2k}}\right] \\
& \stackrel{(c)}{=} \sum^{k}_{i=0}  {2k \choose 2i} \E[s^{2i}_1]  \E[s^{2k- 2i}_2]  \frac{\E[x^{2i}(1)] \E[x^{2k-2i}(2)]}{\E[||\bx||^{2k}]} \\
& \stackrel{(d)}{=} \sum^k_{i=0}   {2k \choose 2i} \frac{(d+2i-2)!! (d+2k-2i-2)!!\cdot (2i-1)!! (2k-2i-1)!!}{(d+2k-2)!! (d-2)!! } \\
& \stackrel{(e)}{=} \frac{(2k)!}{2^k k!} \sum^k_{i=0} {k \choose i} \frac{(d+2i-2)!! (d+2k-2i-2)!!}{(d+2k-2)!! (d-2)!!} .
\end{align*}
 $(a)$ follows from linearity of expectation and the observation above. $(b)$ follows from the independence of $s_1$, $s_2$, and $\bx$. $(d)$ follows from substituting the moments of chi and Gaussian distributions. $(e)$ follows from numerical simplification. 
We now describe the reasoning behind $(c)$. Let $\bz = \frac{\bx ||\by||}{||\bx||}$, where $\by$ and $\bx$ are independent $N(0,I_d)$ random variables. 
By the properties of the Gaussian random variables $\bz$ is also a $N(0,I_d)$ random variable. Thus,
\[
\E[z^{2i}(1)] \E[z^{2k-2i}(2)] = \E \left[ \frac{x^{2i}(1) x^{2k-2i}(2)}{||\bx||^{2k}}\right] \E[||\by||^{2k}].
\]
Rearranging terms, we get 
\[
\E \left[ \frac{x^{2i}(1) x^{2k-2i}(2)}{||\bx||^{2k}}\right] = \frac{\E[z^{2i}(1)] \E[z^{2k-2i}(2)] }{\E[||\by||^{2k}} =  \frac{\E[x^{2i}(1)] \E[x^{2k-2i}(2)]}{\E[||\bx||^{2k}]},
\]
and hence $(c)$. Substituting the above equation in the cosine expansion, we get that the expectation is 
\[
\E[\cos(s_1y(1) + s_2y(2)] ] = \sum^\infty_{k=0} \frac{(-z^2)^k}{k!} \sum^k_{i=0} {k \choose i} \frac{1}{2^k}\frac{(d+2i-2)!! (d+2k-2i-2)!!}{(d+2k-2)!! (d-2)!!}.
\]
Observe that 
\[
\frac{(d+2i-2)!! (d+2k-2i-2)!!}{(d+2k-2)!! (d-2)!!} = \frac{\prod^{k-i-1}_{j=0} (1+2j/d)}{\prod^{k-i-1}_{j=0} (1+2(j+i)/d)},
\]
Hence by Lemma~\ref{lem:tech1},
\begin{align*}
& \left \lvert\frac{\prod^{k-i-1}_{j=0} (1+2j/d)}{\prod^{k-i-1}_{j=0} (1+2(j+i)/d)} -  \left(1  
 + \sum^{k-i-1}_{j=0} \frac{2j}{d} -  \sum^{k-i-1}_{j=0} \frac{2(j+i)}{d}\right) \right \rvert  \\
& \leq \left( \sum^{k-i-1}_{j=0} \frac{2i}{d}  \right)^2 + 
\sum^{k-i-1}_{j=0}
\frac{2i}{d}\left(\frac{2j}{d}\right).
\end{align*}
Simplifying we get,
\begin{align*}
\left \lvert\frac{\prod^{k-i-1}_{j=0} (1+2j/d)}{\prod^{k-i-1}_{j=0} (1+2(j+i)/d)} -  \left(1  + \frac{2i^2 - 2ik}{d}\right) \right \rvert 
 \leq  \frac{4i^2(k-i)^2}{d^2}+ \frac{2i(k-i)(k-i-1)}{d^3}.
\end{align*}
Hence summing over $i$, 
\[
\left \lvert \sum^k_{i=0} {k \choose i} \frac{1}{2^k}\frac{\prod^{k-i-1}_{j=0} (1+2j/d)}{\prod^{k-i-1}_{j=0} (1+2(j+i)/d)} -\left( 1 + \frac{k-k^2}{2d} \right) \right \rvert 
\leq   \frac{k^4}{4d^2}+ \frac{k^2(k-1)}{2d^3}.
\]
Substituting,
\[
\frac{\E[\cos((s_1y(1) + s_2y(2))z]] + \E[\cos((s_1y(1) - s_2y(2))z]]}{2}= \sum^\infty_{k=0}  \frac{(-z^2)^k}{k!}  \left( 1 + \frac{k-k^2}{2d} + c_{k,d} \right),
\]
where $|c_{k,d}| \leq   \frac{k^4}{4d^2}+ \frac{k^2(k-1)}{2d^3}$.
Thus,
\begin{align*}
& \frac{\E[\cos((s_1y(1) + s_2y(2))z]] + \E[\cos((s_1y(1) - s_2y(2))z]]}{2} \\
& = \sum^\infty_{k=0}  \frac{(-z^2)^k}{k!}  \left( 1 + \frac{k-k^2}{2d} + c_{k,d}\right) \\
& \leq \sum^\infty_{k=0}  \frac{(-z^2)^k}{k!}  \left( 1 + \frac{k-k^2}{2d} \right) +  \sum^\infty_{k=0}  \frac{(z^2)^k}{k!} \left(  \frac{k^4}{4d^2}+ \frac{k^2(k-1)}{2d^3} \right)\\
& \leq e^{-z^2} - e^{-z^2} \frac{z^4}{2d} + \frac{e^{z^2}(z^8 + 6z^6 + 7 z^4 + z^2)}{4d^2} + \frac{e^{z^2}z^4(z^6+2z^4)}{2d^3}.
\end{align*}
\end{proof}

\section{Proof of Theorem~\ref{thm:orfd}}
\label{app:orfd}
The proof of the theorem is similar to that of Lemma~\ref{lem:neg_cor} and we outline some key steps. We first bound the bias in Lemma~\ref{lem:orfdbias} and then the variance in Lemma~\ref{lem:orfdvar}.
\begin{lemma}
\label{lem:orfdbias}
If $\bw = \sqrt{d} \by$, where $y$ is distributed uniformly on a unit sphere, then
\[
\left \lvert\E[\cos \bw^T \bz] - \left( e^{-z^2/2} -  e^{-z^2/2} \frac{z^4}{4d} \right)\right \rvert \leq \frac{e^{z^2/2}z^4(z^4 + 8 z^2+ 8)}{16d^2}.
\]
\end{lemma}
\begin{proof}
Without loss of generality, we can assume $\bz$ is along the first coordinate and hence $\bw^T \bz = \sqrt{d}z y(1)$.
A way to compute a uniformly distributed random variable on a sphere with radius $z$ is to generate
$d$ independent random variables $\bx = (x(1),x(2),\ldots, x(d) )$ each distributed $N(0,1)$ and
setting $y(i) = z x(i)/||\bx||$. hence,
\[
\E[\cos \bw^T \bz] = \E \left[ \cos \left(\frac{z\sqrt{d}x(1)}{||\bx||} \right)\right].
\]
The $k^{\text{th}}$ term in the Taylor's series expansion of cosine in the above equation is
\begin{align*}
\frac{(-1)^k}{(2k)!}  \left(\frac{\sqrt{d}x(1)z}{||\bx||} \right)^{2k}
\end{align*}
Similar to the proof of Lemma~\ref{lem:neg_cor}, it can be shown that the expectation of this term is 
\[
\E \left[\frac{(-1)^k}{(2k)!}  \left(\frac{z\sqrt{d}x(1)}{||\bx||} \right)^{2k} \right] = \frac{(-z^2)^k}{2^k k!} \frac{d^{k}}{(d,2k-2)!!}.
\]
Applying Lemma~\ref{lem:tech1} and simplifying,
\begin{align*}
\E[\cos(dy(1))]
& = \sum^{\infty}_{k=0} \frac{(-z^2)^k}{2^k k!} \left( 1 - \frac{k(k-1)}{d} + c'_{k,d} \right),
\end{align*}
where $|c'_{k,d}| \leq \left(\frac{k(k-1)}{d}\right)^2$. Hence, 
\begin{align*}
\left \lvert \E[\cos(dy(1))] -  \sum^{\infty}_{k=0} \frac{(-z^2)^k}{2^k k!} \left( 1 - \frac{k(k-1)}{d}\right) \right \rvert \leq  \sum^{\infty}_{k=0} \frac{(z^2)^k}{2^k k!} \left(\frac{k(k-1)}{d}\right)^2,
\end{align*}
and thus
\[
\left \lvert \E[\cos(dy(1))] -  e^{-z^2/2} +  e^{-z^2/2} \frac{z^4}{4d} \right \rvert \leq \frac{e^{z^2/2}z^4(z^4 + 8 z^2+ 8)}{16d^2}.
\]
\end{proof}
\begin{lemma}
\label{lem:orfdvar}
Let $D \leq d$. If $\bW = \sqrt{d}\bQ$, where $\bQ$ is a uniformly chosen random rotation, then 
\[
\Var\left(\frac{1}{D}\sum^D_{i=1} \cos (\bw^T_i \bz)\right) \leq  \frac{1}{2D} \left((1-e^{-z^2})^2 - \frac{D-1}{d} e^{-z^2} z^4 \right) + \frac{\cO(e^{3z^2})}{d^2}.
\]
\end{lemma}
\begin{proof}
Let $a_i = \cos(\bw^T_i \bz)$. Expanding the variance we have,
\begin{align*}
\Var\left(\frac{1}{D}\sum^D_{i=1} a_i \right) 
& = \frac{1}{D^2}\sum_i \left(\E[a^2_i]  - (\E[a_i])^2\right) +  \frac{1}{D^2}\sum_i \sum_{j \neq i} \left(\E[a_i a_j]  -\E[a_i] \E[a_j] \right) \\
& = \frac{1}{D} \left(\E[a^2_1] -  (\E[a_1]^2)\right) +  \frac{D-1}{D}\left(\E[a_1 a_2]  - \E[a_1] \E[a_2] \right).
\end{align*}
For the first term, rewriting $\cos^2(\bw^T\bz) = \frac{1 + \cos(2 \bw^T\bz)}{2}$, similar to the 
proof of Lemma~\ref{lem:orfdbias} it can be shown that

\begin{align*}
(\E[a^2_1] -  (\E[a_1]^2)\leq \frac{(1-e^{-z^2})^2}{2} + \frac{\cO(e^{3z^2})}{d}.
\end{align*}
Second term can be bounded similar to Lemma~\ref{lem:neg_cor} and here we just sketch an outline.  
Similar to the proof of Lemma~\ref{lem:neg_cor}, the variance boils down to computing the expectation of 
$\cos(\bw^T_1\bz + \bw^T_2 \bz)$. 
\ignore{
\[
\E[ \cos(\bw^T_1\bz + \bw^T_2 \bz)] 
= \E\left[ \frac{\sqrt{d}z(x(1)+x(2))}{||\bx||}\right].
\]
As before the $k^{\text{th}}$ term in the Taylor series can be evaluated to 
\[
\E \left[\left( \frac{(-1)^k}{(2k)!} \frac{\sqrt{d}z(x(1)+x(2))}{||x||}\right)^{2k} \right] 
= \frac{(-z^2)^{k}}{k!} \sum^k_{i=0} {k \choose i} \frac{d^k}{2^kf(d,2k-2)}.
\]
}
Using Lemma~\ref{lem:tech1} and summing Taylor's series we get 
\begin{align*}
\left \lvert \E[ \cos(\bw^T_1\bz + \bw^T_2 \bz)] - e^{-z^2} +  e^{-z^2} \frac{z^4}{d}  \right \rvert 
\leq  \frac{e^{z^2}z^4(z^4 + 4z^2 + 2)}{d^2}.
\end{align*}
Substituting the above bound and the expectation from Lemma~\ref{lem:orfdbias}, we get
\[
\E[a_1 a_2]  - \E[a_1] \E[a_2] \leq -e^{z^2}\frac{z^4}{2d} + \frac{\cO(e^{3z^2})}{d^2},
\]
and hence the lemma.
\end{proof}
\ignore{
\begin{align*}
\E[\cos(s_1y(1) + s_2y(2)]] 
&= \sum^\infty_{k=0}  \frac{(-z^2)^k}{k!}  \left( 1 + \frac{k-k^2}{2d} + o \left( \frac{1}{d}  \right) \right)  + 2 e^{-k_0} \\
& = \sum^\infty_{k=0}  \frac{(-z^2)^k}{k!}  \left( 1 + \frac{k-k^2}{2d} \right)  + 2 e^{-k_0} + o \left( \frac{e^{z^2}}{d} \right)  \\
& = e^{-z^2} - e^{-z^2} \frac{z^4}{2d} + 2 e^{-k_0} + o \left( \frac{e^{z^2}}{d} \right).
\end{align*}
}
\ignore{
Combining the above set of equations, results in
\begin{lemma}
Under the proposed scheme,
\[
\Var\left(\frac{1}{D}\sum^D_{i=1} \cos(w^T_i z) \right) \leq \frac{1}{2D} \left((1-e^{-z^2})^2 - \frac{D-1}{d}  e^{-z^2} z^4 \right) + o(1/d).
\]
\end{lemma}
Thus when $D \propto d$, we gain significantly for small values of $z$.
}

\section{Proof of Theorem \ref{thm:biassorf}}
\label{app:thm:biassorf}
The proof follows from the following two technical lemmas.
\begin{lemma}
\label{lem:stein}
Let $z'$ be distributed according to $N(0,||\bx||^2_2)$ and $y' = \sum^d_{i=1} x(i) d_i$, where $d_i$s are independent Rademacher random variables. For any function $g$ such that $|g'|\leq 1$ and $|g| \leq 1$,
\[
|\E[g(z')] - \E[g(y')]| \leq \frac{3}{2}\sum^d_{i=1}  \frac{x^3(i)}{||\bx||^2_2}.
\]
\end{lemma}
\begin{proof}
Let $z = z'/||\bx||_2$, $y = y'/||\bx||_2$, and $h(x)  = g (||\bx||_2 x)$, for all $x$. Hence $h(z) = g(z')$ and $h(y) = g(y')$. By a lemma due to Stein~\cite{chat07},
\begin{align*}
|\E[g(z')] - \E[g(y')]|  
&= |\E[h(z)] - \E[h(y)]| \\
&\leq \text{sup}_{f} \{ |\E[f'(y) -  y f(y)]| : |f|_\infty \leq ||\bx||_2, |f'|_\infty \leq \sqrt{2/\pi} ||\bx||_2, |f''|_\infty \leq 2 ||\bx||_2\}.
\end{align*}
We now bound the term on the right hand side by classic Stein-type arguments.
\[
\E[y f(y)] = \sum^d_{i=1} \frac{x(i)d_i}{||\bx||_2} \E[f(y)].
\]
Let $y_i = y - \frac{x(i)d_i}{||\bx||_2}$. Observe that
\begin{align*}
\E[d_i f(y)] 
&= \E[d_i (f(y) - f(y_i))] \\
& =\E[d_i( f(y) - f(y_i)) - d_i(y - y_i) f'(y_i)]
+ \E[d_i (y- y_i) f'(y_i)],
\end{align*}
where the first equality follows from the fact that $y_i$ and $d_i$ are independent and $d_i$ has zero mean.
By Taylor series approximation, the first term is bounded by 
\[
 \lvert \E[d_i f(y) - f(y_i) - d_i(y - y_i) f'(y_i)] \rvert 
 \leq \frac{1}{2} (y-y_i)^2 |f''|_\infty = \frac{1}{2} \frac{x^2(i)}{||\bx||^2_2} |f''|_\infty.
 \]
 Similarly,
 \[
 \E[d_i (y- y_i) f'(y_i)] = \frac{x(i)}{||\bx||_2}f'(y_i).
 \]
Combining the above four equations, we get 
\[
\left \lvert  \E\left[yf(y) - \sum^d_{i=1} \frac{x^2(i)}{||\bx||^2_2} f'(y_i)\right] \right \rvert \leq \sum^d_{i=1} \frac{|x^3(i)|}{||\bx||^3_2} |f''|_\infty.
\]
Similarly, note that 
\[
\left \lvert  \E\left[f'(y)- \sum^d_{i=1} \frac{x^2(i)}{||\bx||^2_2} f'(y_i) \right] \right \rvert
 \leq \sum^d_{i=1} |f''|_\infty \frac{x^2(i)}{||\bx||^2_2}	
 \E[ |y -y_i|]
 = \sum^d_{i=1} |f''|_\infty \frac{|x^3(i)|}{||\bx||^3_2}.
\]
Combining the above two equations, we get 
\[
||\E[yf(y) - f'(y))]| \leq \frac{3|f''|_\infty}{2} \sum^d_{i=1}  \frac{|x^3(i)|}{||\bx||^3_2}.
\]
Substituting the bound on the second moment of $f$ yields the result.
\end{proof}
Let $\bG$ be a random matrix with i.i.d. $N(0,1)$ entries as before. 
Using the above lemma we show that $ \sqrt{d} \bH\bD_1 \bH \bD_2$ behaves like $\bG$ while computing the bias.
\begin{lemma}
For a given $\bx$, let $\bz = \bG \bx$ and $\by =  \sqrt{d} \bH\bD_1 \bH \bD_2 \bx$. For any function $g$ such that $|g'|\leq 1$ and $|g| \leq 1$,
\[
\left\lvert\frac{1}{d} \sum^d_{i=1}\E\left[ g(z(i))\right] - \frac{1}{d}\sum^d_{i=1} \E\left[g(y(i))\right]\right \rvert \leq 6\frac{\norm{\bx}_2}{\sqrt{d}}.
\]
\end{lemma}
\begin{proof}
By triangle inequality,
\[
\left\lvert\frac{1}{d}\sum^d_{i=1}\E\left[ g(z(i))\right] - \frac{1}{d}\sum^d_{i=1} \E\left[g(y(i))\right]\right\rvert \leq \frac{1}{d} \sum^d_{i=1} |\E[g(z(i))] - \E[g(y(i))]|.
\]
Let $\bu = \bH\bD_2 \bx$. Then for every $i$, $y(i) = \sum_{j} H(i,j) D_2(j) u(j)$. Hence by Lemma~\ref{lem:stein}, we can relate expectation under $y$ to the expectation under Gaussian distribution:
\begin{align*}
|\E[g(z(i))] - \E[g(y(i))]|  
&\stackrel{(a)}{=} |\E[\E[g(z(i))] - \E[g(y(i)) |\bu]]| \\
& \leq \frac{3}{2}\sum^d_{i=1} \E\left[ \frac{|u^3(i)|}{||\bu||^2_2} \right] \\
&= \frac{3}{2}\sum^d_{i=1} \E\left[ \frac{|u^3(i)|}{||\bx||^2_2} \right],
\end{align*}
where the last equality follows from the fact that $\bH\bD_2$ does not change rotation and $(a)$ follows from the law of total expectation.
By Cauchy-Schwartz inequality, for each $i$
\[
\E[|u(i)|^3] \leq \sqrt{\E[u^6(i)]},
\]
It can be shown that
\begin{align*}
\E[u^6(i)] \leq \frac{15||\bx||^6_2}{d^3},
\end{align*}
Summing over all the indices yields the lemma. 
\end{proof}
Theorem~\ref{thm:biassorf} follows from the Bochner's theorem and the fact that $\cos(\cdot)$ satisfies requirements for the above lemma. We note that Theorem~\ref{thm:biassorf} holds for the matrix $ \sqrt{d} \bH\bD_1 \bH \bD_2$ itself and the third component $\bH \bD_3$ is not necessary to bound the bias.
\section{Proof of Theorem~\ref{thm:nearortho}}
\label{app:thm:nearortho}
To prove Theorem~\ref{thm:nearortho}, we use the Hanson-Wright Inequality. 
\ignore{
The following bound follows from the union bound and McDiarmid's inequality.
\begin{lemma}
\label{hdconc}
Let $\bu = \bH \bD \bz$. Then with probability at least $ 1- 2 d e^{-d\epsilon^2/2}$,
\[
\max_{1\leq i \leq d} |u(i)| \leq \epsilon ||z||_2.
\]
\end{lemma}
\begin{definition}
\label{subgaussian}
We call random variable $Z$ subgaussian if $\mathbb{E}|Z|^{p} = O(p)^{\frac{p}{2}}$ as $p \rightarrow \infty$.
\end{definition}
Note that clearly a random variable $d_{i}$ defined as:
\[
d_{i} =
\left\{
	\begin{array}{ll}
		1  & \mbox{with probability $\frac{1}{2}$} , \\
		-1 & \mbox{otherwise} 
	\end{array}
\right.
\]
is subgaussian.
}
\begin{lemma}[Hanson-Wright Inequality]
\label{hanson_wright_theorem}
Let $\textbf{X}=(X_{1},...,X_{n}) \in \mathbb{R}^{n}$ be a random vector with independent subgaussian components $X_{i}$ which satisfy: $\mathbb{E}[X_{i}] = 0$ and $\|X_{i}\|_{sg} \leq K$ for some constant $K>0$. Let $\textbf{A} \in \mathbb{R}^{n \times n}$. Then for any $t > 0$ the following holds:
\[
\mathbb{P}[|\textbf{X}^{T}\textbf{AX} - \mathbb{E}[\textbf{X}^{T}\textbf{AX}]| > t]
\leq 2e^{-c \min(\frac{t^{2}}{K^{4}\|\textbf{A}\|^{2}_{F}},\frac{t}{K^{2}\|\textbf{A}\|_{2}})},
\]
for some universal positive constant $c>0$.
\end{lemma}
\begin{proof}[Proof of Theorem~\ref{thm:nearortho}]
For a vector $\bu$, let $\text{diag}(\bu)$ denote the diagonal matrix whose entries correspond to the 
entries of $\bu$. For a diagonal matrix $\bD$, let $\text{vec}(\bD)$ denote the vector corresponding to the diagonal entries of $\bD$. Let $\bv = \bH \bD_3\bz$  and $\bu = \bH \bD_2 \bv= \bH \text{diag}(\bv) \text{vec}(\bD_2)$. Observe that 
\[
\sqrt{d} \bH \bD_1 \bH \bD_2 \bH \bD_3 \bz = \sqrt{d} \bH \text{diag}( \bH \bD_2 \bH \bD_3 \bz ) \text{vec}(\bD_1).
\]
Hence $\tilde\bR =  \sqrt{d} \bH \text{diag}( \bH \bD_2 \bH \bD_3 \bz )$.  Note that all the entries of $\sqrt{d}\bH$ have magnitude $1$ and $\bH \bD_2 \bH \bD_3 $ do not change norm of the vector. Hence, each row of $\tilde \bR$ has norm $||\bz||_2$. 
To prove the orthogonality of rows of $\tilde \bR$, we need to show that for any $i$ and $j \neq i$,
\[
\sqrt{d}\sum^d_{k=1}H(i,k) H(j,k) u^2(k)
\]
is small. We first show that the expectation of the above quantity is $0$ and then use the Hanson-Wright inequality to prove concentration.
Let $\bA$ be a diagonal matrix with $k^{th}$ entry being $\sqrt{d} H(i,k) H(j,k)$. The above equation can be rewriten as
\[
\sum^d_{k=1} H(i,k) H(j,k) u^2(k) =\text{vec}(\bD_2)^T  \text{diag}(\bv)\bH^T  \bA \bH \text{diag}(\bv) \text{vec}(\bD_2).
\]
Observe that the $(l,l)$ entry of the $\bH^T  \bA \bH $ is 
\begin{align*}
\sum^d_{k=1}H^T(l,k) A(k,k) H(k,l) 
& = \sum^d_{k=1}H(k,l) A(k,k) H(k,l) \\
& = \frac{1}{d} \sum^d_{k=1}A(k,k)\\
&= \sum^d_{k=1} H(i,k) H(j,k) = 0,
\end{align*}
where the last equality follows from observing that the rows of $\bH$ are orthogonal to each other. Together with the fact that elements of $\bD_2$ are independent of each other, we get
\[
\E[ \bu^T \bA \bu] = \E[\text{vec}(\bD_2)^T  \text{diag}(\bv)\bH^T  \bA \bH \text{diag}(\bv) \text{vec}(\bD_2)] = 0,
\]
To prove the concentration result, observe that the entries of $\text{vec}(\bD_2)$ are independent and sub-Gaussian, and hence we can use the Hanson-Wright inequality. To this end, we bound the Frobenius and the spectral norm of the underlying matrix. For the Frobenius norm, observe that
\begin{align*}
|| \text{diag}(\bv)\bH^T  \bA \bH \text{diag}(\bv) ||_F 
& \stackrel{(a)}{\leq} \left(||\bv||_\infty \right)^4 ||\bH^T  \bA \bH||_F  \\
& \stackrel{(b)}{=}  \left(||\bv||_\infty \right)^4 ||\bA||_F \\
& \stackrel{(c)}{=} d \left(||\bv||_\infty \right)^4,
\end{align*}
where $(a)$ follows by observing that each $\text{diag}(\bv)$ changes the Frobenius norm by at most $||\bv||^2_\infty $,
$(b)$ follows from the fact that $\bH$ does not change the Frobenius norm, and $(c)$ follows by substituting $\bA$. 

To bound the spectral norm, observe that 
\begin{align*}
|| \text{diag}(\bv)\bH^T  \bA \bH \text{diag}(\bv) ||_2
& \stackrel{(a)}{\leq} \left(||\bv||_\infty \right)^2 ||\bH^T  \bA \bH||_2  \\
& \stackrel{(b)}{=}  \left(||\bv||_\infty \right)^2 ||\bA||_2 \\
& \stackrel{(c)}{=}  \left(||\bv||_\infty \right)^2,
\end{align*}
where $(a)$ follows by observing that each $\text{diag}(\bv)$ changes the spectral norm by at most $||\bv||_\infty $,
$(b)$ follows from the fact that rotation does not change the spectral norm, and $(c)$ follows by substituting $\bA$. 
Since $\bv = \bH \bD_3 \bz$, by McDiarmid's inequality, it can be shown that with probability $\geq 1 - 2de ^{-d\epsilon^2/2}$, $||\bv||_\infty\leq \epsilon ||\bz||_2$.  Hence, by the Hanson-Wright inequality, we get 
\[
\text{Pr}\left(\sqrt{d}\sum^d_{k=1} H(i,k) H(j,k) u^2(k) > t||\bz||_2\right) \leq  2de^{- d\epsilon^2/2} + 2e^{-c \min(t^2/(d\epsilon^4),t/\epsilon^2)},
\]
where $c$ is a constant. Choosing $\epsilon = ({t/d})^{1/3}$ results in the theorem.
\end{proof}

\begin{figure}[t]
\centering\subfigure[$D = d / 2$]{\includegraphics[width=0.33\textwidth]{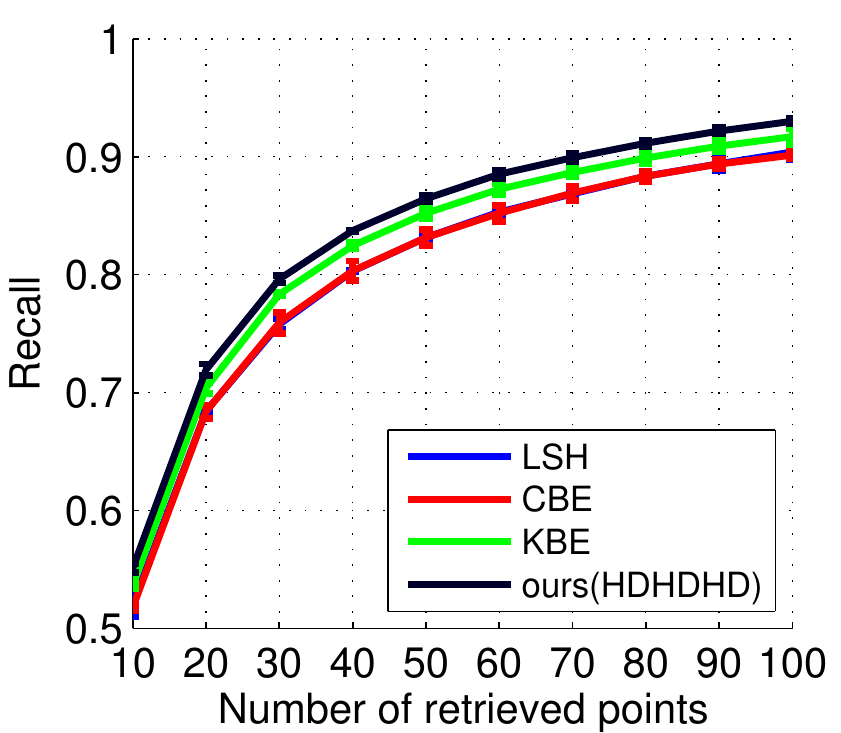}}
\hspace{-0.23cm}
\subfigure[$D = d$]{\includegraphics[width=0.33\textwidth]{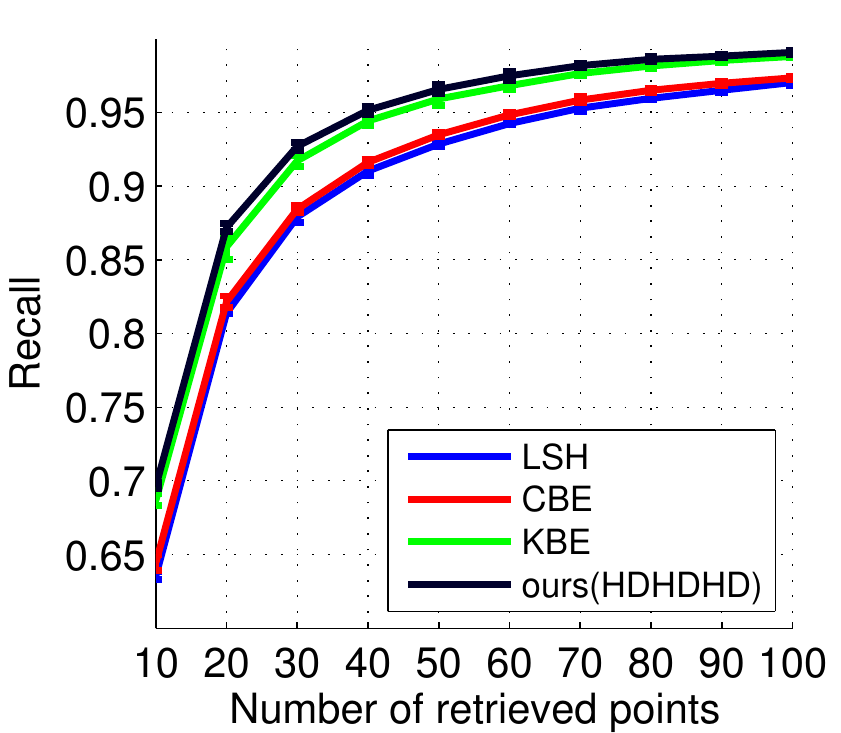}}
\hspace{-0.23cm}
\subfigure[MSE]{\includegraphics[width=0.33\textwidth]{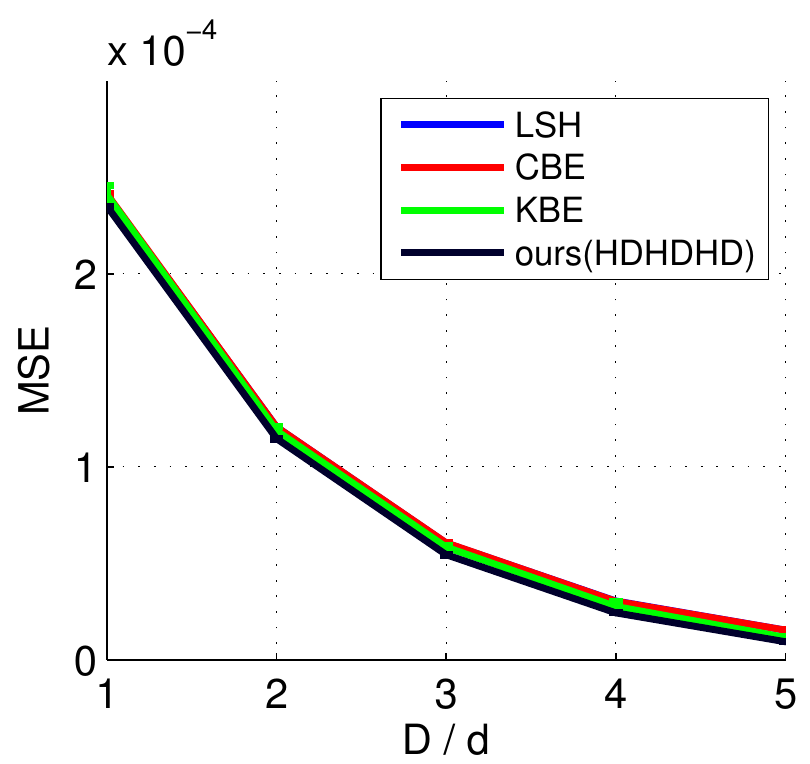}}
\caption{Recall and angular MSE on a 16384-dimensional dataset of natural images \cite{zhang2015fast}.}
\label{fig:binary_embedding}
\end{figure}

\section{Discrete Hadamard-Diagonal Structure in Binary Embedding}
Motivated by the recent advances in using structured matrices in binary embedding, we show empirically that the same type of structured discrete orthogonal matrices (three blocks of Hadamard-Diagonal matrices) can also be applied to approximate angular distances for high-dimensional data. 
Let $\bW \in \mathbb{R}^{D\times d}$ be a random matrix with i.i.d.\ normally distributed entries. The classic Locality Sensitive Hashing (LSH) result shows that the $\sign$ nonlinear map $\phi: \phi(\bx) = \frac{1}{\sqrt{D}} \sign(\bW \bx)$ can be used to approximate the angle, i.e., for any $\bx, \by \in \mathbb{R}^d$
\[
\quad \phi(\bx)^T \phi(\by) \approx \theta(\bx, \by) / \pi. 
\]

We compare random projection based Locality Sensitive Hashing (LSH) \cite{charikar2002similarity}, Circulant Binary Embedding (CBE) \cite{icml14_cbe} and Kronecker Binary Embedding (KBE) \cite{zhang2015fast}. We closely follow the experimental settings of \cite{zhang2015fast}. We choose to compare with  \cite{zhang2015fast} because it proposed to use another type of structured random orthogonal matrix (Kronecker product of orthogonal matrices). As shown in Figure \ref{fig:binary_embedding}, our result (HDHDHD)  provides higher recall and lower angular MSE in comparison with other methods. 
\end{appendices}
\end{document}